\documentclass{article}

\usepackage{amsmath,amsfonts,bm}

\def\eqref#1{equation~\ref{#1}}

\def\1{\bm{1}}

\def\bv{{\bm{v}}}

\def\bI{{\bm{I}}}

\DeclareMathAlphabet{\mathsfit}{\encodingdefault}{\sfdefault}{m}{sl}
\SetMathAlphabet{\mathsfit}{bold}{\encodingdefault}{\sfdefault}{bx}{n}

\def\gC{{\mathcal{C}}}
\def\gD{{\mathcal{D}}}
\def\gE{{\mathcal{E}}}

\def\gN{{\mathcal{N}}}

\def\gP{{\mathcal{P}}}

\def\gW{{\mathcal{W}}}

\def\od{{\mathrm{d}}}

\def\sfD{{\mathsf{D}}}

\def\sfW{{\mathsf{W}}}

\newcommand{\E}{{\mathbb{E}}}

\newcommand{\odt}{{\mathrm{d}t}}
\newcommand{\odx}{{\mathrm{d}x}}
\newcommand{\ody}{{\mathrm{d}y}}

\usepackage{microtype}
\usepackage{graphicx}
\usepackage{subcaption}
\usepackage{booktabs} %

\usepackage{hyperref}

\usepackage[preprint]{icml2026}

\usepackage{amsmath}
\usepackage{amssymb}
\usepackage{mathtools}
\usepackage{amsthm}

\usepackage{array}
\usepackage{multirow}

\theoremstyle{plain}
\newtheorem{theorem}{Theorem}[section]
\newtheorem{proposition}[theorem]{Proposition}
\newtheorem{lemma}[theorem]{Lemma}

\theoremstyle{definition}
\newtheorem{definition}[theorem]{Definition}
\newtheorem{assumption}[theorem]{Assumption}
\theoremstyle{remark}

\usepackage[capitalize,noabbrev]{cleveref}

\crefname{section}{Section}{Sections}
\crefname{theorem}{Theorem}{Theorems}
\crefname{lemma}{Lemma}{Lemmas}
\crefname{equation}{Eq.}{Eq.}
\crefname{proposition}{Proposition}{Propositions}
\crefname{claim}{Claim}{Claims}
\crefname{remark}{Remark}{Remarks}
\crefname{assumption}{Assumption}{Assumptions}
\crefname{definition}{Definition}{Definitions}
\crefname{appendix}{Appendix}{Appendices}
\crefname{algorithm}{Algorithm}{Algorithms}
\crefname{figure}{Figure}{Figures}
\crefname{table}{Table}{Tables}
\crefname{example}{Example}{Examples}

\usepackage[disable,textsize=tiny]{todonotes}

\usepackage{subcaption}
\usepackage{pifont}%
\usepackage{placeins}
\usepackage{tabularx}
\usepackage{booktabs} %
\usepackage{soul}
\usepackage{multirow}
\usepackage{adjustbox}
\usepackage{xspace} 
\usepackage{colortbl}
\usepackage{amsthm}
\usepackage[inkscapelatex=false]{svg}

\makeatletter
\DeclareRobustCommand\onedot{\futurelet\@let@token\@onedot}
\def\@onedot{\ifx\@let@token.\else.\null\fi\xspace}
\def\eg{\emph{e.g}\onedot} 
\def\ie{\emph{i.e}\onedot}

\makeatother

\newcommand{\methodname}{RelaxFlow\xspace}
\newcommand{\methodabbr}{RelaxFlow\xspace}
\newcommand{\method}{\methodname}
\newcommand{\task}{text-driven amodal 3D generation\xspace}

\newcommand{\extremeocc}{ExtremeOcc-3D\xspace}
\newcommand{\unseen}{AmbiSem-3D\xspace}
\newcommand{\obsbranch}{Observation Branch\xspace}
\newcommand{\sembranch}{Semantic-Prior Branch\xspace}
\newcommand{\lprelax}{low-pass relaxation\xspace}
\newcommand{\semcorridor}{semantic corridor\xspace}

\definecolor{myorange}{HTML}{B13B12}
\definecolor{myblue}{HTML}{14A0EC}
\definecolor{mygreen}{HTML}{7ED569}
\definecolor{mygred}{HTML}{F8D5D5}

\begin{document}

\twocolumn[
  \icmltitle{RelaxFlow: Text-Driven Amodal 3D Generation}

  \icmlsetsymbol{equal}{*}

  \begin{icmlauthorlist}
    \icmlauthor{Jiayin Zhu}{nus}
    \icmlauthor{Guoji Fu}{nus}
    \icmlauthor{Xiaolu Liu}{zju,nus}
    \icmlauthor{Qiyuan He}{nus}
    \icmlauthor{Yicong Li}{ustc}
    \icmlauthor{Angela Yao}{nus}
    \\ \vspace{2px}\url{https://github.com/viridityzhu/RelaxFlow}
  \end{icmlauthorlist}

  \icmlaffiliation{nus}{National University of Singapore}
  \icmlaffiliation{zju}{Zhejiang University}
  \icmlaffiliation{ustc}{University of Science and Technology of China}

  \icmlcorrespondingauthor{Yicong Li}{liyicong@ustc.edu.cn}

  \icmlkeywords{Image-to-3D Generation, Amodal 3D Generation, Text-Guided 3D Generation}

  \vskip 0.3in
]

\printAffiliationsAndNotice{}  %

\begin{abstract}

Image-to-3D generation faces inherent semantic ambiguity under occlusion, where partial observation alone is often insufficient to determine object category. 
In this work, we formalize \emph{\task}, where text prompts steer the completion of unseen regions while strictly preserving input observation. Crucially, we identify that these objectives demand distinct control granularities: rigid control for the observation versus relaxed structural control for the prompt. To this end, we propose \textbf{\methodname}, a training-free dual-branch framework that decouples control granularity via a Multi-Prior Consensus Module and a Relaxation Mechanism. Theoretically, we prove that our relaxation is equivalent to applying a low-pass filter on the generative vector field, which suppresses high-frequency instance details to isolate geometric structure that accommodates the observation. To facilitate evaluation, we introduce two diagnostic benchmarks, \textbf{\extremeocc} and \textbf{\unseen}. Extensive experiments demonstrate that \methodabbr successfully steers the generation of unseen regions to match the prompt intent without compromising visual fidelity.

\end{abstract}

\begin{figure}
    \centering
    \includegraphics[width=\linewidth]{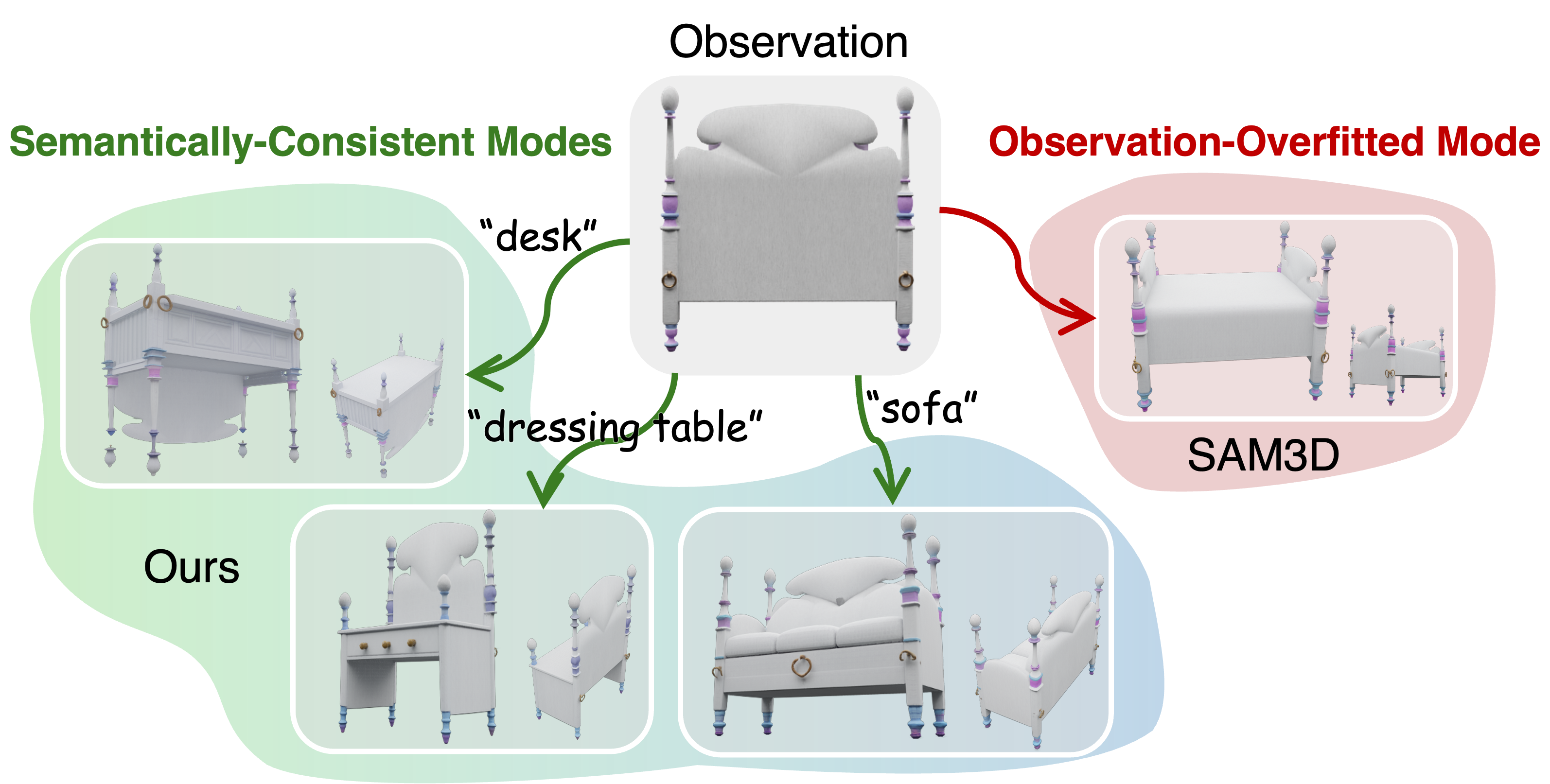}
    \caption{The case for multiple plausible amodal 3D interpretations under occlusion. Feedforward image-to-3D model (\eg, SAM3D) collapses to a single overfitted bed-like shape,
while our RelaxFlow resolves the ambiguity via \emph{text-driven amodal generation}, allowing users to inject explicit textual intent and steer the generation toward alternative \emph{semantically consistent} amodal 3D shapes.
    }
    \vspace{-1em}
    \label{fig:teaser}
\end{figure}

\section{Introduction}
\label{introduction}

Humans possess the natural ability to perceive objects’ complete physical structure even under occlusion. This cognitive capacity, known as amodal perception, allows us to mentally complete the object by supplementing the visible fragments with internal prior knowledge.
As a core technology for AR/VR and robotics~\cite{xiu2025egotwin}, image-to-3D generation~\cite{xiang2024trellis,zhao2025hunyuan3d2,team2025sam3d,Chen_2025_CVPR_3DTopiaXL,Liu2023Zero1to3,ye_hi3dgen_2025,wu_direct3d_2024,poole2023dreamfusion} seeks to emulate this capability. However, unlike human cognition, which draws on rich internal context, existing generative models typically rely solely on the visible pixels. Consequently, when observation is occluded, this lack of semantic guidance leads to severe ambiguity.

In this work, we move beyond generic single-condition generation and define \textit{\task}, where the user's text prompt explicitly steers the completion of unseen parts while preserving the observed evidence. 
Taking Figure~\ref{fig:teaser} as an illustration, when only a wooden backboard is visible, the amodal object remains ambiguous. The unseen regions could plausibly correspond to a sofa, a bed, or a dressing table.
Previous feedforward image-to-3D models, like SAM3D, often collapse to an \emph{observation-overfitted} generation (here, a bed-like shape) by uncontrolled implicit hallucination.
Under our task formulation, a user can explicitly resolve this ambiguity via a text prompt (\eg, specifying ``a sofa''), while keeping the visible part consistent with the input observation.

To achieve this controllability, the generation must satisfy two constraints simultaneously: (1) \textbf{observation fidelity}, where the visible regions must strictly adhere to the pixel-level details of the input observation; and (2) \textbf{prompt following}, where the unobserved regions must conform to the user's prompt, while maintaining structural coherence with the observation. State-of-the-art feedforward models~\cite{xiang2024trellis,team2025sam3d,zhao2025hunyuan3d2} prioritize observation fidelity but often trap the output in an ``\emph{observation-overfitted mode}'' (see Figure~\ref{fig:teaser}, ``bed" result) without prompt-following capabilities. 
Conversely, optimization-based methods~\cite{instructnerf2023,kamata_i323_2023,miao2025rethinkingsds,zhu2026AnchorDS} enforce strong {prompt-following}, but tend to over-smooth or distort the observation evidence, as gradients from the semantic prior clash with the hard pixel reconstruction constraints~\cite{He2023T3BenchBC}. As a result, the tension between the {observation fidelity and prompt following} cannot be resolved by existing strategies.

Existing strategies attempt to enforce both constraints under a uniform control granularity. However, this approach creates an inherent conflict: the \textit{input observation} demands strict adherence to ensure visual fidelity, whereas the \textit{input prompt} serves only as a relaxed structural guide to accommodate the observation. To resolve this tension, our core insight is that the generation process must decouple these control granularities. Accordingly, we adopt a dual-branch strategy: an \obsbranch for strict adherence, and a \sembranch that captures global semantics while tolerating local variations. Finally, we employ a visibility-aware fusion mechanism to ensure the semantic guide steers only the occluded regions, while the observation strictly governs the visible pixels.

While the \obsbranch can be effectively managed via standard rigid feature injection, the challenge arises in designing the {relaxation mechanism} in the \sembranch.  
To achieve this, we adopt a training-free strategy with two core designs. We first introduce a Multi-Prior Consensus module that retrieves a set of reference images via the text prompt. Since these references share a semantic category but differ in appearance, their cross-attention naturally amplifies structural consensus while suppressing inconsistent instance-specific textures. We then apply a \lprelax by smoothing cross-attention logits within the generation backbone. 

Our theoretical analysis shows that this smoothing is equivalent to low-pass filtering the underlying generative vector field.  The filtering suppresses high-frequency instance identities and exposes a coarse \semcorridor that enforces only low-frequency global geometry. As a result, the \sembranch achieves the target geometry (\eg, ensuring the shape of a ``sofa'') but remains agnostic to high-frequency details, thereby tolerating local variations and accommodating the specific textures of the observation.

To facilitate systematic evaluation, we develop
two diagnostic benchmarks for \task, focusing on whether text can disambiguate unseen structure without sacrificing adherence to the observed evidence.
\emph{\extremeocc} targets {extreme occlusion} in natural indoor scenes, where visible evidence is often insufficient to infer even the object category, and a text label is required to resolve the completion.
\emph{\unseen} targets {semantic branching}: each ambiguous input image admits multiple plausible interpretations, paired with textual prompts that specify distinct intended outcomes under the same visual evidence.
Results on these benchmarks show that our method produces high-quality 3D assets that faithfully reflect the user's semantic intent while strictly anchoring the visible evidence, and remain 3D plausible and consistent.

Our contributions are fourfold:
    (1) We formalize \task, a new setting that requires resolving occlusion-induced ambiguity via text prompts while strictly preserving the input observation;
    (2) We propose \methodabbr, a training-free dual-branch framework that effectively decouples control granularity via consensus-based multi-prior conditioning and visibility-aware fusion;
    (3) We provide a theoretical proof that our relaxation mechanism on the semantic branch is equivalent to a low-pass filter on the generative vector field, justifying our strategy for extracting structural guidance. 
    (4) We introduce two diagnostic benchmarks, \extremeocc and \unseen. Extensive experiments demonstrate that our approach significantly outperforms existing strategies in steering unseen geometry without compromising fidelity.

\section{Related Work}
\label{sec:related_work}

\noindent\textbf{Occlusion-aware 3D Generation.}
Feedforward image-to-3D methods~\cite{xiang2024trellis,zhao2025hunyuan3d2,Chen_2025_CVPR_3DTopiaXL,lan2024ln3diff,xu2024instantmesh} enable efficient single-view generation, but often rely on strong observation priors and degrade under heavy occlusion, where multiple 3D hypotheses remain consistent with the visible evidence. This is closely related to amodal completion, which recovers occluded object extent~\cite{li2025intermediate}. In 2D, \citet{ao_2025_cvpr_openWorldAmodal,xiu2025geometric} present training-free text-guided amodal completion; Pix2Gestalt~\cite{ozguroglu_2024_cvpr_pix2gestalt} completes in 2D and lifts to 3D via reconstruction.
Occlusion completion has also been integrated into feedforward image-to-3D. Amodal3R~\cite{Wu_2025_ICCV_amodal3r} and DeOcc-1-to-3~\cite{qu_2025_deocc123} finetune feedforward models with masked inputs, while SAM3D~\cite{team2025sam3d} scales training for stronger completion. However, these approaches require retraining and typically learn a dataset-driven ``most likely'' completion, limiting controllability when multiple amodal interpretations are plausible. Achieving retraining-free external semantic control (text or images) while preserving observation fidelity remains insufficiently addressed, especially under extreme occlusion.

\noindent\textbf{Multi-condition Control in 3D Generation.}
A common way to reduce ambiguity is to provide more views, \ie, multi-view-to-3D methods~\cite{Liu2023Zero1to3,shi2024mvdream}; however, they assume mutually consistent views of the same instance, rather than an auxiliary condition that encodes semantic intention. Text-guided 3D editing~\cite{instructnerf2023,kamata_i323_2023,miao2025rethinkingsds,parelli20253dlatte,li2024laso, li2025intermediate} offers language control but typically requires an existing 3D asset and per-instance optimization or inversion, differing from feedforward generation conditioned on an observed image plus an intention signal. More broadly, feedforward image-to-3D generators~\cite{team2025sam3d,zhao2025hunyuan3d2,xiang2024trellis,lan2024ln3diff} are not designed for unified text-and-image control, making joint alignment of observation evidence and free-form instructions challenging. A practical alternative is to use text-to-image models~\cite{cai2025zimage,labs2025flux1kontext,podell2023sdxl} to synthesize visual proxies of text intentions that can condition or guide 3D generation.

\noindent\textbf{Composing and Steering Guidance in Flow Models.}
Classifier-free guidance (CFG)~\cite{ho2022cfg} amplifies conditional signals during diffusion sampling and is often extended to weighted compositions of multiple prompts. Methods that modify condition-dependent dynamics, such as FlowEdit~\cite{kulikov2025flowedit}, further suggest that composing conditional vector fields can enforce multiple constraints~\cite{liu2022compositionaldiffusion}.
Orthogonally, \citet{hong2024seg} analyzes attention blur and uses it as regularizing guidance, and \citet{sadat_2025_guidanceInFrequency} studies frequency-aware guidance and sampling. However, a systematic link between attention smoothing and an explicit low-pass semantic prior for stable, training-free controllable generation under severe underdetermination remains underexplored.

\section{ODE Flow Formulation for Dual-Branch 3D Generation}\label{sec:method}

\subsection{ODE Flow Formulations}\label{sec:ode}

\textbf{Motivation: The dual-objective inference problem.} 
We view \task as a dual-objective inference problem involving two distinct physical forces: 
\begin{enumerate}%
    \setlength{\itemsep}{0pt}
    \item \emph{Observation Constraint:} 
        A hard force that anchors the generation to the visible pixels, ensuring reconstruction fidelity.
    \item \emph{Semantic Prior:} 
        A soft force that completes the unobserved regions, governed by the user's high-level semantic intent rather than pixel evidence.
\end{enumerate}
 Under external blockers or self-occlusion, the same observation can admit multiple plausible 3D completions, so the semantic completion cannot be uniquely determined from the observation alone.

\textbf{Information gap.}
The observation condition $c_{\text{obs}}$ (the input image) is sufficient to enforce the \emph{hard} constraint term, but is generally \emph{insufficient} to uniquely determine the semantic prior in occluded regions. We therefore introduce an auxiliary latent condition $c_{\text{sem}}$ (the user's true intent) that is not a function of $c_{\text{obs}}$ in ambiguous cases.

Motivated by these, we formulate the governing ODE for all methods as an interpolation where the semantic term depends on a variable semantic prior $c^{\star}$: 
\begin{equation}
    \frac{\od x_t}{\od t} 
    \!=\! 
    (1 \!-\! \alpha_t)
    \underbrace{\bv_{\rm obs}(x_t, t, c_{\text{obs}})}_{\text{Observation Constraint}} 
    \!+ 
    \alpha_t
    \underbrace{\bv_{\rm prior}(x_t, t, c^{\star})}_{\text{Intent Prior (semantic)}}.
    \label{eq:general_flow}
\end{equation}
Here, $\alpha_t$ schedules how strongly the intent prior influences the trajectory. The distinction between methods lies in different choices of (i) the conditioning $c^{\star}$ used by the prior branch, and (ii) the parameterization of the prior-induced vector field $\bv$.

\textbf{Three instantiations.}
\emph{(i) Oracle trajectory.} The ideal flow uses a perfect decomposition: an observation-preserving field and an oracle semantic field driven by the true intent tokens $c^{\star}=c_{\text{sem}}$. This trajectory has access to the missing information needed to resolve occlusion-induced ambiguity. For the \emph{oracle} semantic transport field, we denote $v_{\rm prior}(\cdot,t,c^\star)=v_{\rm sem}(\cdot,t,c_{\rm sem})$.

\emph{(ii) Standard neural flow}. Feedforward image-to-3D models typically reuse observation tokens for everything, effectively setting $c^{\star}=c_{\text{obs}}$. This conflates ``what must be preserved'' with ``what is missing,'' and can yield \emph{observation-overfitted} solutions when the unseen regions are underdetermined. Here, we denote $v_{\rm prior}(\cdot,t,c^\star)=\bar v_{\theta}(\cdot,t,c_{\rm obs})$.

\emph{(iii) \methodabbr \;(ours).}
We keep the observation branch unchanged to preserve evidence, but replace the prior branch with an \emph{explicit intent prior} represented by visual tokens derived from the text prompt, \ie, $c^{\star}=c_{\text{prior}}$ (Sec.~\ref{sec:multiprior}).
Crucially, we \emph{relax} the prior-conditioned estimator with a low-pass operator $\mathcal{R}_\sigma$ (Eq.~\ref{eq:lp_relax_def}) to suppress instance-specific, high-frequency nuisance that destabilizes joint guidance; in practice $\mathcal{R}_\sigma$ is realized by attention-logit blurring (Eq.~\ref{eq:attn_blur}).
This decoupling bridges the information gap without retraining the generator.

Compared to the standard neural flow, \methodabbr changes only the \emph{source} and \emph{smoothness} of the semantic term: it is driven by intent priors rather than observation tokens, and is low-pass–relaxed to remain compatible with the observation constraint. The next section shows that this relaxation provably reduces semantic estimation error and tightens a stability bound.

\begin{figure*}[t]
        \centering
        \includegraphics[width=\linewidth]{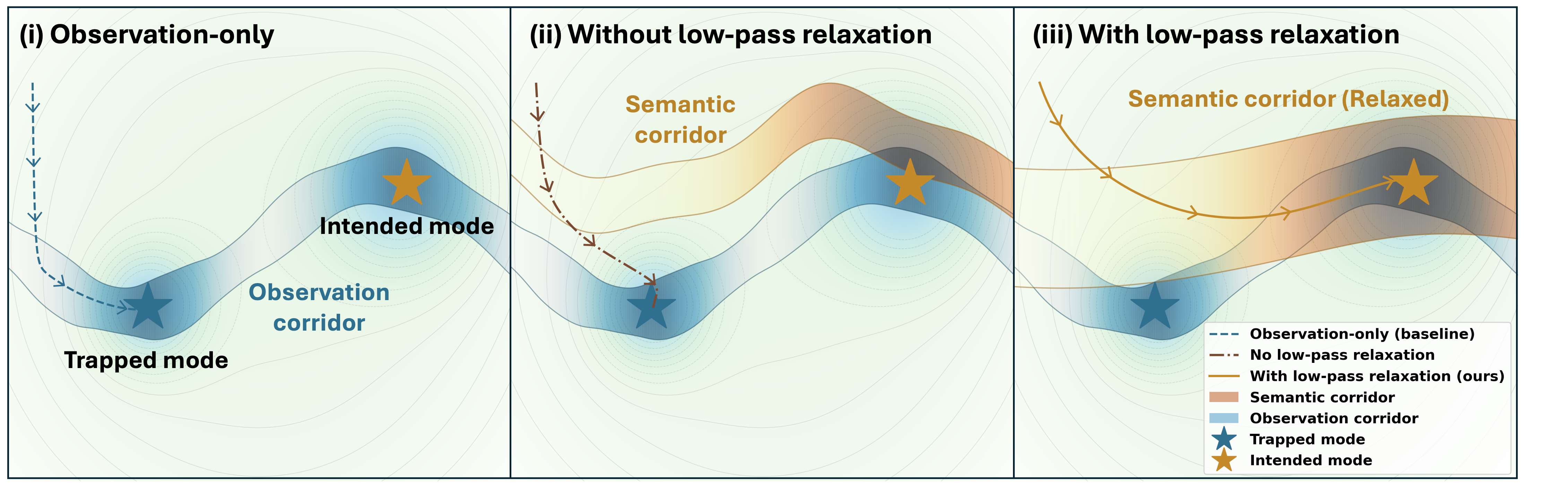}
        \caption{ \textbf{Conceptual illustration of low-pass relaxation.}  The background depicts a conceptual \emph{compatibility landscape} over latent states, where high density corresponds to low ``energy''. Star markers indicate a spurious trapped mode and the intended target mode. Throughout, we use \emph{corridor} to denote a tube of latent states traced by integral curves of a conditioned velocity field while satisfying a constraint.
        Smoothing the prior-conditioned guidance thickens the semantic corridor and steers trajectories toward the intended mode while remaining compatible with the observation corridor.}
        \label{fig:method-demon}
\end{figure*}

\subsection{Theoretical Justification: Low-Pass Relaxation and Stability}\label{sec:lowpass_theory}
\paragraph{Low-pass relaxation operator.}
Our analysis treats \emph{low-pass relaxation} as an operator applied to the \emph{prior-conditioned} semantic estimator.
Let $v_\theta(x,t,c_{\rm prior})$ denote the generator's (prior-conditioned) velocity estimate.
We define its $\sigma$-relaxed counterpart by
\begin{equation}
\tilde v_\theta(\cdot,t,c_{\rm prior})
\;:=\;
\mathcal{R}_\sigma\!\left[v_\theta(\cdot,t,c_{\rm prior})\right],
\label{eq:lp_relax_def}
\end{equation}
where $\mathcal{R}_\sigma$ is a Gaussian low-pass operator that attenuates high-frequency components and preserves low-frequency structure (formalized in Appendix~\ref{sec:app:spectral}).
In the main text, we analyze $\mathcal{R}_\sigma$ abstractly at the velocity-field level; in implementation, we realize $\mathcal{R}_\sigma$ by blurring the \emph{prior-branch} cross-attention logits (Eq.~\ref{eq:attn_blur}), which induces an analogous relaxation on the resulting velocity field (Appendix~\ref{appen:logit-smooth-as-vector-relax}).

\subsubsection{Spectral analysis of semantic guidance}

We illustrate the intuition in Figure~\ref{fig:method-demon}, where the landscape background serves as a conceptual visualization of compatibility. %
In panel (i), the blue band shows the \emph{observation corridor} induced by $c_{\rm obs}$, and the dashed curve is an observation-only rollout constrained to remain within this tube. Panels (ii--iii) additionally introduce an intent prior $c_{\rm prior}$ (orange) to resolve ambiguity in occluded regions: with a raw prior-conditioned estimator (ii), instance-specific high-frequency nuisance can intermittently point against observation-consistent directions, fragmenting the effective semantic corridor and causing unstable dynamics; with low-pass relaxation (iii), these high-frequency components are suppressed, producing a thicker, smoother semantic corridor that stays compatible with the observation corridor and reliably steers the trajectory toward the intended mode.

Motivated by this corridor-fragmentation viewpoint, we model the semantic branch as providing a \emph{transport or steering velocity} in the generator's ODE. 
We posit that the oracle semantic transport field $\bv_{\rm sem}$, which captures global structures like the shape of a ``bed" or ``sofa", is inherently band-limited to low frequencies. In contrast, errors in the learned semantic velocity (\eg, texture-driven conflicts or instance-specific hallucinations) behave as high-frequency noise (see \cref{assump:sem-tube,assump:hf-mismatch} in Appendix~\ref{sec:app:spectral} for formal spectral definitions).

Based on this spectral separation, we treat $\tilde v_\theta = \mathcal{R}_\sigma[v_\theta]$ (Eq.~\ref{eq:lp_relax_def}) as a low-pass relaxation:
it suppresses high-frequency, instance-specific error while preserving low-frequency semantic transport.
We measure semantic estimation error along a reference trajectory $X_t$ by the $L_2$ path norm:
\begin{equation}
    \gE_{\rm sem}^2(v)
    \;:=\;
    \int_0^1
    \big\| \bv_{\rm sem}(X_t,t,c_{\rm sem}) - v(X_t,t)\big\|^2 \,\mathrm{d}t
    ,
    \label{eq:esem_def}
\end{equation}
where $v(X_t,t)$ denotes the semantic-branch estimator used by a method (standard or ours).
As derived in \cref{thrm:low-pass}, this operation strictly reduces the semantic estimation error:
\begin{equation}
    \gE_{\rm sem}(\tilde{\bv}_{\theta}) < \gE_{\rm sem}(\bar{\bv}_{\theta}).
\end{equation}
This denoising operation ensures that the prior branch provides stable structural guidance without injecting the high-frequency conflicts that typically disrupt observation constraints, matching the corridor-smoothing effect in Figure~\ref{fig:method-demon}. 
We refer to Appendix~\ref{sec:app:spectral} for the detailed derivations.

\begin{figure*}[t]
        \centering
        \includegraphics[width=\linewidth]{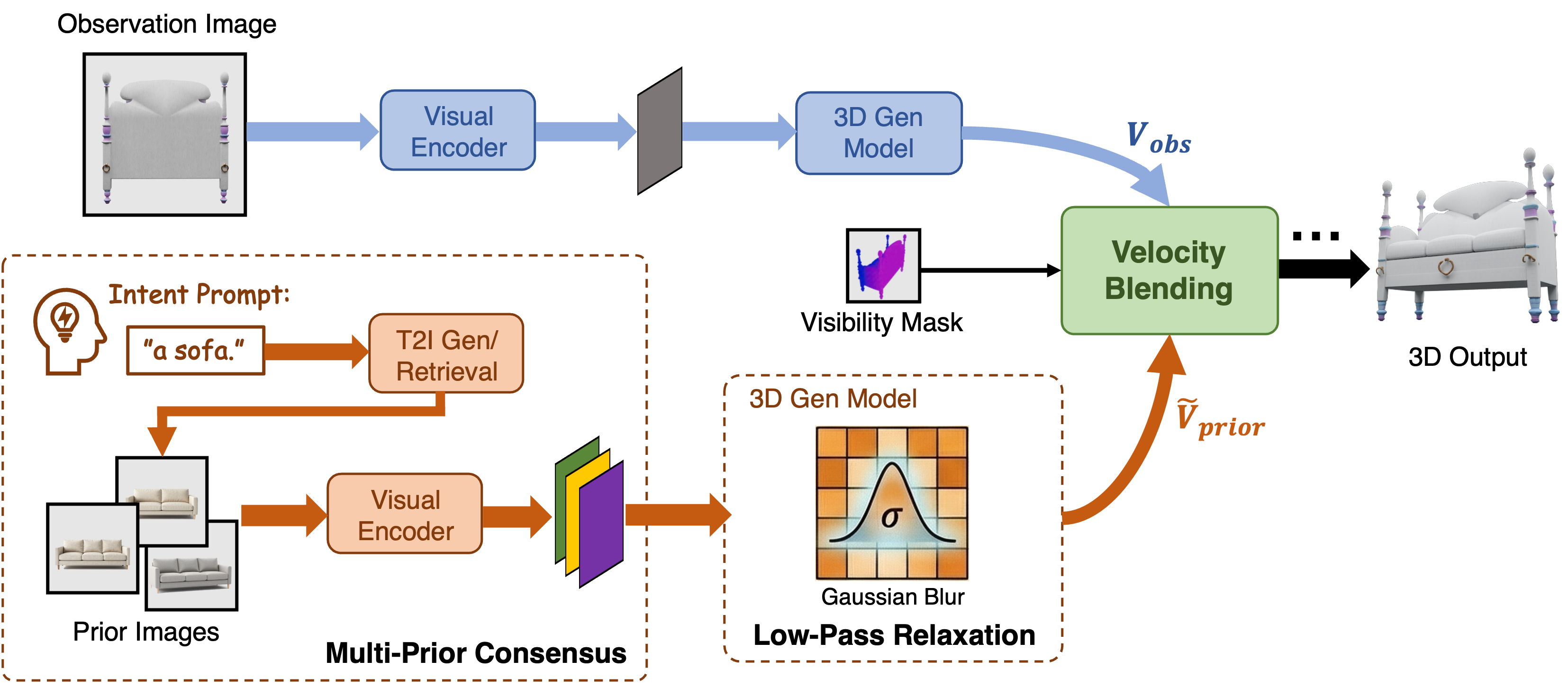}
        \caption{  \textbf{\methodabbr\ pipeline overview}. An observation-driven branch preserves visible evidence, while the semantic guidance is injected through the intent prompt via multi-prior consensus and low-pass relaxation. Dual branches are fused via velocity blending to resolve occlusion-induced ambiguity.}
        \label{fig:method-pipeline}
\end{figure*}

\subsubsection{Stability and Wasserstein Bounds}
We translate the reduction in vector field error into a guarantee on generation quality. We prove that by filtering high-frequency semantic noise, the generated distribution becomes geometrically closer to the ground truth.

We now translate semantic-field estimation error into a distributional guarantee. We use $\gE_{\rm sem}(\cdot)$ 
as defined in \cref{eq:esem_def}. %
To measure the quality of the final output, we use the \emph{Wasserstein-2 ($\gW_2$) distance} to the ground truth distribution. 
We choose $\gW_2$ because it captures the \textit{geometric transport cost} between distributions~\citep{heusel2017gans}, ensuring that generated shapes are geometrically plausible rather than just statistically overlapping. 

Relying on the Lipschitz continuity of the neural estimator (see \cref{assump:Lip:nn} in Appendix), we derive an upper bound for the generation quality. 
As proven in \cref{thrm:Wasserstein-bound:main} (Appendix), the distance to the ground truth is bounded by 
\begin{equation}
    \gW_2(p, \hat{p}) \leq C \cdot \left(\gE_{\rm obs}(\bar{\bv}_{\theta}) + \gE_{\rm sem}(\tilde{\bv}_{\theta}) + \delta_{\rm prior} \right).
\end{equation}
The term $\delta_{\rm prior}$ captures the residual mismatch between the visual prior tokens $c_{\rm prior}$ and the oracle intent $c_{\rm sem}$ (\ie, the gap introduced by using proxy priors). Low-pass relaxation tightens the bound by reducing the semantic estimation error $\gE_{\rm sem}^2(\tilde{\bv}_\theta)$. In Sec.~\ref{sec:multiprior}, we further reduce $\delta_{\rm prior}$ in practice by using \emph{multiple} priors and letting attention implement a consensus over consistent attributes.

\section{RelaxFlow Framework}

We introduce \methodabbr, a training-free dual-branch inference framework for \task that realizes the theoretical flow derived in Sec.~\ref{sec:ode}.
As shown in Figure~\ref{fig:method-pipeline}, our pipeline decouples the generation into two parallel trajectories:
(1) An \emph{Observation Branch} driven by $c_{\text{obs}}$ to preserve high-frequency visible evidence; and
(2) A \emph{Semantic-Prior Branch} driven by $c_{\text{prior}}$ that applies low-pass relaxation to steer global structure.
We describe the construction of the prior condition $c_{\text{prior}}$ (Sec.~\ref{sec:multiprior}), the realization of the relaxed vector field $\tilde{\bv}$, %
and the fusion strategy of the two branches (Sec.~\ref{sec:fusion}).

\subsection{Backbone and Inference Setup}\label{sec:backbone}

We instantiate \methodabbr on two common feedforward image-to-3D generators, TRELLIS~\cite{xiang2024trellis} and SAM3D~\cite{team2025sam3d}, both of which follow a two-stage pipeline: \emph{Sparse Structure} (SS) for generating geometry, and \emph{Structured Latent} (SLAT) for refining texture and detail.
The base generators of both SS and SLAT are conditional rectified-flow models. During inference, we solve its ODE with an explicit solver. With $K$ steps and schedule $\{t_k\}$, Euler updates take the form $x_{k+1}=x_k+\Delta t\,v(x_k,t_k,c)$, where $c=E(I,M)$ denotes conditioning tokens encoded from an image-mask pair via a visual encoder $E$, and $v$ is the model-predicted velocity.
In SS, a flow generator outputs a latent $z^{\text{ss}}$ decoded into an occupancy volume $S=\sfD_{\text{ss}}(z^{\text{ss}})$ on a $64^3$ grid, from which we extract occupied voxel indices
\begin{equation}
\gC=\{i \mid S_i>0\}.
\end{equation}
In SLAT, a second flow generator outputs per-voxel features $z^{\text{slat}}\in\mathbb{R}^{|\gC|\times 8}$, and a decoder maps these latents to Gaussian splats or a mesh representation.

\subsection{Multi-Prior Consensus}\label{sec:multiprior}

A key practical challenge in \task is that modern feedforward 3D generators are typically \emph{visual-token conditioned}~\cite{he2025rear}: their conditioning interface expects image-derived tokens rather than free-form language. Directly injecting a text embedding would require aligning an external language space to the generator's internal visual tokens (\eg, via adapters or finetuning), which is both model-specific and risks introducing distribution mismatch. We instead convert the user's intent prompt into \emph{visual proxies} and keep \emph{all} conditions in the generator's native visual-token space.

\paragraph{Prompt-to-prior as a visual interface.}
Given an intent prompt $p$, we construct a small set of $N$ prior images $\{(I_p^n, M_p^n)\}_{n=1}^N$ using either retrieval from a data pool or an off-the-shelf text-to-image generator to sample multiple instances. These priors are not meant to match the observation's instance appearance; rather, they provide \emph{structural semantic evidence} for the occluded or ambiguous regions. The same mechanism also supports user-provided reference images by directly treating them as priors.

\paragraph{Consensus over multiple priors.}
A single prior image can entangle the intended attribute with incidental, instance-specific texture and style details. 
We therefore condition on a \emph{set} of priors that share the intended attribute but vary in irrelevant factors (\eg, given a prompt ``a bird with a red beak'', we can generate multiple images with red beaks but different body shapes and feather colors). The model should preserve what is consistent across priors and ignore conflicts. 
Concretely, we concatenate the per-prior token sequences into one long sequence and feed it to cross-attention in a single pass. This induces a simple consensus effect: attributes that are consistent across priors appear repeatedly in the token set and thus receive more aggregate attention, while idiosyncratic or conflicting details tend to be diluted. 
Empirically, this reduces the proxy gap $\delta_{\rm prior}$ in Sec.~\ref{sec:lowpass_theory} by making the effective conditioning closer to the user’s intent than any single prior.

\subsection{Dual-Branch Sampling as ODE Interpolation}\label{sec:dualbranch}
\methodabbr implements Eq.~\ref{eq:general_flow} by coupling two branches without retraining, implemented via a prior-guided sampling rule. 
At each solver step $k$, we evaluate both branches on the \emph{same} state $x_k$:
\begin{equation}
	    v_{\text{obs}} = v_\theta(x_k, t_k, c_{\rm obs}), \quad
	    \tilde{v}_{\text{prior}} = \tilde{v}_\theta(x_k, t_k, c_{\rm prior}),
\end{equation}
where $\tilde{v}_\theta$ denotes evaluation under \lprelax, which we implement by smoothing the cross-attention logits as in Eq.~\ref{eq:attn_blur}. 
This shared-state design matters: it avoids inconsistent trajectories (\eg, a prior-only path that drifts away from the observed geometry), and instead combines two compatible velocities.

\paragraph{Realizing $\tilde{v}$ via logit smoothing.}
To instantiate the abstract relaxation $\tilde v_\theta=\mathcal{R}_\sigma[v_\theta]$ (Eq.~\ref{eq:lp_relax_def}) in transformer-based generators, we apply $\mathcal{R}_\sigma$ to the \emph{prior-branch} cross-attention \emph{logits}.
This blurring suppresses token-local, high-frequency sensitivity in prior conditioning, yielding the smoother semantic corridor in Figure~\ref{fig:method-demon}.
For an attention head with {logits $L_{i,j}=q_i^\top k_j/\sqrt d$, where $(i,j)$ follow the underlying 2D/3D token grid order, we apply a separable Gaussian filter $G_\sigma$} before the softmax:
\begin{equation}\label{eq:attn_blur}
    \tilde{L} = G_\sigma * L
    :=
    \big(G_\sigma^{(q)} *_{q} L\big) *_{k} G_\sigma^{(k)},
\end{equation}
and $\text{Attn}(Q,K,V) = \text{softmax}(\tilde{L})\,V$.
Here $*_{q}$ and $*_{k}$ denote 1D convolution along the query and key indices, respectively.
In implementation, $G_\sigma$ is a discrete 1D Gaussian {kernel with $\pm 3\sigma$ support}, normalized to sum to $1$; the operation is realized as two sequential 1D convolutions along the sequence dimensions ($Q$ and $K$), equivalent to a separable 2D blur over the logit matrix~\cite{hong2024seg}. The hyperparameter $\sigma$ controls the smoothing strength, and we empirically show that performance is robust to $\sigma$ around our default choice ($\sigma=1.0$) (Sec.~\ref{subsec:exp_ablation}).
\obsbranch\ attention is left unmodified.

\paragraph{Fusion Strategy.}\label{sec:fusion}
To combine the branches, we employ a time-dependent and spatially-aware fusion scheme.

Prior works~\cite{huang_dreamtime_2023, choi2022perception,zhu2024InstructHumans} show that early steps determine the global semantic mode under ambiguity, so prior guidance is most useful there; later steps are where the model refines geometry and texture details, where we prefer to rely on the observation branch to avoid overpainting and to preserve evidence-consistent high frequencies. As a result, the prior steers {where needed} (coarse, ambiguous completion) and yields to the observation for {what must be preserved} (visible structure and fine details).
We then discretize the time-dependent gate $\alpha_t$ as $\alpha_k$ and apply Euler updates:
\begin{equation}\label{eq:linear-blend}
	    x_{k+1} %
	    = x_k + \Delta t \, \big( (1-\alpha_k)v_{\text{obs}} + \alpha_k \tilde{v}_{\text{prior}} \big),
\end{equation}
where $\alpha_k = g(k,K)$ is discretized and clamped to $[0,1]$, and we use a simple {linear cutoff} schedule:
\begin{equation}\label{eq:alpha-schedule}
	    \alpha_k =
	    \begin{cases}
	        1 - k/K, & k \le \lfloor \rho K \rfloor, \\
	        0, & \text{otherwise},
	    \end{cases}
\end{equation}
where $\rho$ is the cutoff threshold. 
{Intuitively}, this schedule allows the prior $c_{\rm prior}$ to steer the global semantic mode early, then yield to the \obsbranch with only observation evidence $c_{\rm obs}$ to refine fine-grained details.

In the SLAT stage, injecting the \sembranch uniformly risks overpainting regions that are directly supported by the observation. We therefore estimate per-voxel visibility from the observation camera and use it to spatially separate the conditions: observed surfaces are preserved while the prior steers only genuinely occluded regions.
Given the object pose from the backbone, we project voxel centers to the image, build a z-buffer depth map $D'$, and compute an occlusion margin $\Delta_i = z_i - D'(u_i,v_i)$ for each voxel $i$ projected to pixel $(u_i,v_i)$ with depth $z_i$. This margin is converted into a soft visibility weight $m_i \in (0,1]$ via a Gaussian falloff (details in Appendix~\ref{app:vis-mask}). 
We apply the visibility mask to the \emph{velocity interpolation} in SLAT:
\begin{equation}
v_i = v_{\text{obs},i} + (1-m_i)\,\alpha_k\,(\tilde v_{\text{prior},i}-v_{\text{obs},i}),
\end{equation}
and update $x_{k+1}=x_k+\Delta t\,v$ as in Eq.~\ref{eq:linear-blend}. If $m_i\!\approx\!1$ (visible), then $v_i\!\approx\!v_{\text{obs},i}$; if $m_i\!\approx\!0$ (occluded), the step follows the time-varying blend controlled by $\alpha_k$.

\section{Experiments}
\label{sec:experiments}

\begin{figure*}[h!]
    \centering
    \includegraphics[width=1\linewidth]{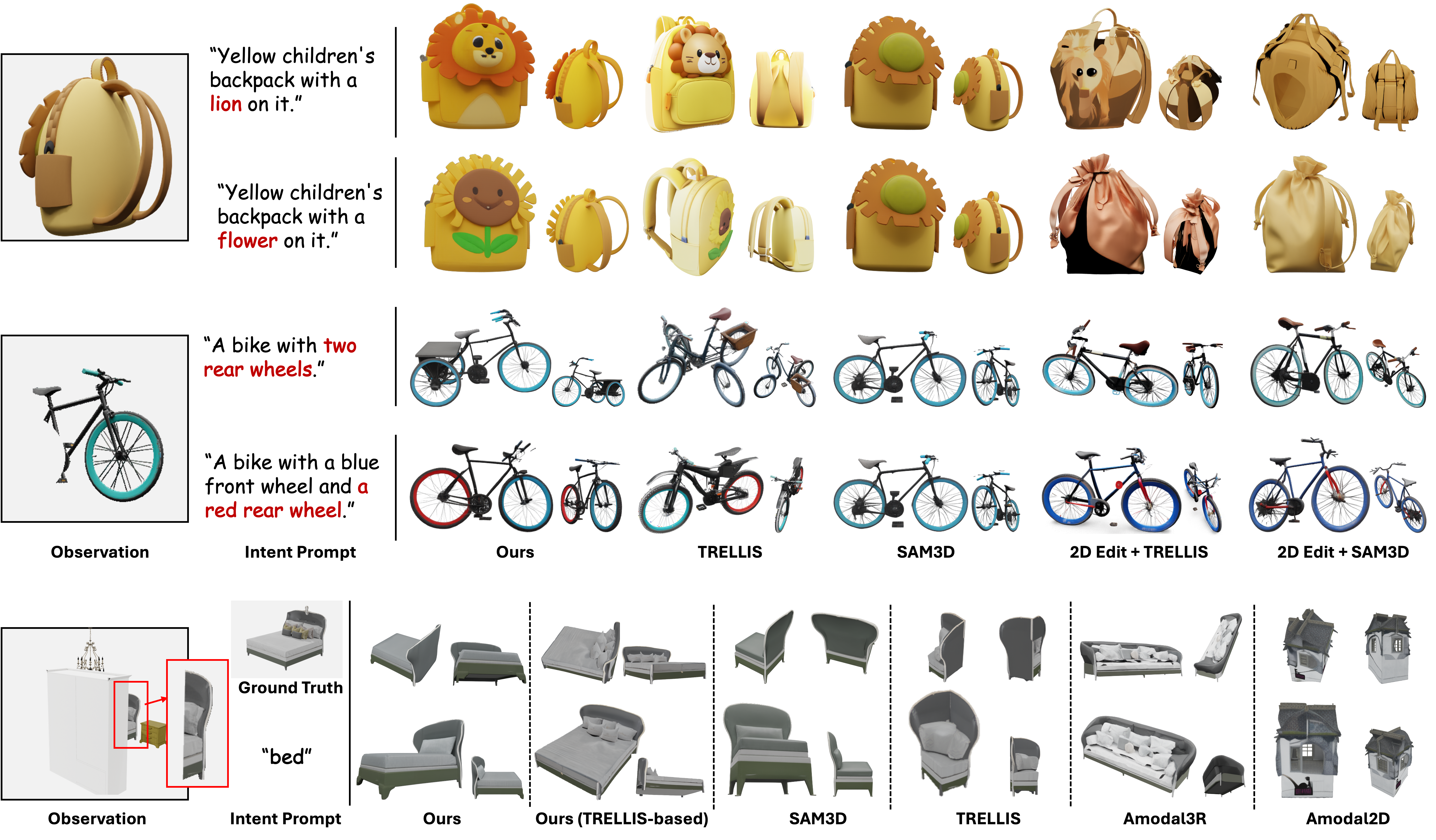}
    \caption{\textbf{Qualitative Comparisons.} \textbf{Top:} Comparisons on \unseen examples. Each contains two different intent prompts for the same observation.  Our method preserves observation fidelity while enabling prompt-controlled completion. \textbf{Bottom:} The case in \extremeocc with an intent prompt ``bed''. Under extreme occlusion, baselines either overfit the visible region or yield implausible shapes, whereas ours follows the category prior while maintaining visible evidence.}
    \label{fig:qual-all}
\end{figure*}

We evaluate \methodabbr on two complementary tasks: (i) amodal 3D completion under extreme occlusion (\extremeocc), and (ii) \task under semantic branching (\unseen). Both benchmarks are curated for controlled evaluation; full details about dataset curation are in Appendix~\ref{app:dataset-curation}.

\subsection{Task Setting}

\noindent\textbf{\extremeocc.}
We build \extremeocc from 3D-FUTURE/3D-FRONT indoor scenes~\cite{fu_2021_ijcv_3dfuture,fu_2021_iccv_3dfront}, where occlusions arise naturally from realistic furniture arrangements and camera viewpoints. Using ground-truth meshes and poses, we compute an \emph{occlusion ratio} $r$ for each object. When $r \ge 80\%$, the visible evidence is often insufficient even for humans to identify the object type. We filter $264$ test cases at this threshold, each paired with a category-level text prior. The task is to generate a complete 3D asset faithful to the visible evidence while conforming to the specified category.

\noindent\textbf{\unseen.}
\unseen tests whether text can resolve \emph{semantic uncertainty} rather than merely refine appearance. We curate $21$ ambiguous cases from ObjaverseXL~\cite{deitke2023objaverse} exhibiting three ambiguity types: masked semantic completion, view-induced identity ambiguity, and intrinsic shape ambiguity. Each case has a single input image and multiple text branches, where each branch corresponds to a distinct but plausible semantic completion. We note that \unseen is intentionally multi-solution and lacks a unique 3D ground truth.

\noindent\textbf{Metrics.}
We evaluate observation preservation (min-LPIPS), semantic alignment (CLIP-Score~\cite{icmlRadfordKHRGASAM21clipscore}), and multi-view realism (FID~\cite{nipsHeuselRUNH17fid}) in 2D. In 3D, as strict geometry errors are ill-suited for our task, where multiple completions are all considered valid, we instead evaluate set-level \textit{semantic} similarity via Point-FID~\cite{alex2022pointe}, which leverages the Point-E encoder to extract semantic features of point clouds. Details of metric computation are provided in Appendix~\ref{app:metric-details}.

\subsection{Implementation Details}
\label{subsec:implementation-details}
We evaluate \methodabbr as a plug-and-play module on SAM3D~\cite{team2025sam3d} and TRELLIS~\cite{xiang2024trellis}. For TRELLIS (no pose estimation), we disable the visibility-aware mask and apply other components unchanged.
On \extremeocc, we apply \methodabbr only at the geometry stage; on \unseen, we apply it to both stages since semantics affect both structure and appearance.
We use $\sigma{=}1.0$, $\rho{=}0.2$, and $N{=}3$ prior images. All experiments run on a single NVIDIA A40 GPU.

\noindent\textbf{Prior image sources.}
For \extremeocc, we retrieve prior images by randomly sampling a different object of the same category from the dataset.
For \unseen, we generate prior images from the text prompt using Z-Image~\cite{cai2025zimage}; all compared methods receive the same generated priors for fair comparison.
Details on prior construction are provided in Appendix~\ref{app:prior-construction}.

\setlength{\tabcolsep}{1pt}
\begin{table}[t!]
\renewcommand{\arraystretch}{1.2}
\caption{\textbf{Quantitative comparison on \extremeocc.} Methods are grouped by occlusion awareness. Our method improves over both backbones across all metrics.} \label{tab:comparison_occ}
\scriptsize
\scalebox{0.95}{
    \begin{tabular}{c|l|ccccc}
    \toprule
    Category & Method & {CLIP$_\text{img}$$\uparrow$} & {CLIP$_\text{txt}$$\uparrow$} & {FID$\downarrow$} & {LPIPS$\downarrow$} & {Point-FID$\downarrow$} \\
    \midrule
    \multirow{2}{*}{\shortstack{Non-Occlusion\\Aware}} 
     & TRELLIS & 0.78 & 23.14 & 122.68 & 0.83 & 141.48 \\
     & \textbf{Ours (w/ TRELLIS)} & \textbf{0.80} & \textbf{24.09} & \textbf{100.75} & \textbf{0.80} & \textbf{97.79} \\
    \midrule
    \multirow{4}{*}{\shortstack{Occlusion-\\Aware}} 
    & Amodal2D+SAM3D & 0.76 & 21.59 & 94.38 & 0.56 & 127.27 \\
     & Amodal3R & 0.77 & 22.29 & 118.49 & 0.60 & 129.46 \\
     & SAM3D & 0.84 & 24.08 & 50.73 & 0.54 & 100.38 \\
     & \textbf{Ours (w/ SAM3D)} & \textbf{0.87} & \textbf{27.26} & \textbf{39.44} & \textbf{0.51} & \textbf{81.11} \\
    \bottomrule
    \end{tabular}
}
\end{table}

\subsection{Amodal 3D Completion Under Extreme Occlusion}
\label{subsec:exp_results_occ}

\noindent\textbf{Baselines.}
We compare against: (1) non-occlusion-aware methods (TRELLIS~\cite{xiang2024trellis}), (2) 2D completion pipelines (Amodal2D~\cite{ao_2025_cvpr_openWorldAmodal}+SAM3D), and (3) occlusion-aware methods (SAM3D~\cite{team2025sam3d}, Amodal3R~\cite{Wu_2025_ICCV_amodal3r}).

\noindent\textbf{Results.}
Table~\ref{tab:comparison_occ} shows that \methodabbr delivers consistent gains over its backbone models. On TRELLIS, Point-FID decreases from 141.48 to 97.79. On SAM3D, \methodabbr achieves the best overall performance: CLIP$_\text{txt}$ increases from 24.08 to 27.26, and Point-FID decreases from 100.38 to 81.11. Notably, CLIP$_\text{img}$ and LPIPS remain competitive with SAM3D, indicating that \methodabbr preserves the observed evidence in the input view rather than drifting from it. Meanwhile, the lower FID and Point-FID further suggest improved global consistency across views and higher overall 3D quality. Figure~\ref{fig:qual-all} provides qualitative comparisons: TRELLIS exhibits noticeable drift from the observation, SAM3D responds weakly to the textual prompt, and 2D-based editing pipelines often introduce geometric artifacts. In contrast, \methodabbr preserves the observed structure while following the intended semantics. Overall, these results indicate that \methodabbr mitigates observation-overfitted collapse under extreme occlusion.

\subsection{Text-driven Amodal 3D Generation Under Semantic Branching}
\label{subsec:exp_results_unseen}

\noindent\textbf{Baselines.}
We compare against SAM3D~\cite{team2025sam3d}, TRELLIS~\cite{xiang2024trellis}, and 2D-edit-then-3D pipelines, where SDXL~\cite{podell2023sdxl} editing is followed by TRELLIS or SAM3D reconstruction. Since TRELLIS accepts multi-view images, we provide the prior image as a second view. Appendix~\ref{app:baseline-details} further compares with commercial 2D editing, video generation, multi-view generation, and direct text+image 3D pipelines.

\noindent\textbf{Results.}
Table~\ref{tab:comparison_unseen} and Figure~\ref{fig:qual-all} show that \methodabbr achieves the best alignment to both observation (CLIP$_\text{img}$) and intent prompt (CLIP$_\text{txt}$). Multi-view TRELLIS often drifts from the observation due to conflicting views, while 2D-editing pipelines introduce geometry-inconsistent artifacts.

\noindent\textbf{User Study.}
Since \unseen allows intent prompt-driven generations which do not have unique ground truths, we conduct a user study with $32$ volunteers evaluating all $21$ cases on \emph{Text--Image Alignment} and \emph{3D Fidelity}. As shown in Table~\ref{tab:comparison_unseen} (right), \methodabbr is preferred by a large margin (68.52\% overall), confirming its ability to resolve ambiguity while preserving observed evidence.

\setlength{\tabcolsep}{1pt}
\begin{table}[t]
\caption{\textbf{Quantitative comparison on \unseen.} Left: automatic metrics (CLIP similarities). Right: user study results ($n{=}32$). Our method achieves the best scores on both evaluations.} \label{tab:comparison_unseen}
\scriptsize
\renewcommand{\arraystretch}{1.2}
\scalebox{0.98}{
    \begin{tabular}{>{\centering\arraybackslash}m{2.2cm}|cc|ccc}
    \toprule
\multirow{2}{*}{\textbf{Method}} 
& \multicolumn{2}{c|}{\textbf{CLIP Score}} 
& \multicolumn{3}{c}{\textbf{User Study}} \\
\cmidrule(lr){2-3} \cmidrule(lr){4-6}
& {CLIP$_\text{img}$$\uparrow$}  
& {CLIP$_\text{txt}$$\uparrow$}  
& {Alignment$\uparrow$} 
& {3D Fidelity$\uparrow$} 
& {Overall Pref.$\uparrow$} \\
\midrule
    SAM3D & 0.85 & 26.29 & 4.84\% & 13.59\% & 9.22\% \\
    TRELLIS (multi-view) & 0.80 & 26.59 & 3.75\% & 8.28\% & 6.02\% \\
    SDXL + TRELLIS & 0.81 & 26.76 & 6.09\% & 8.91\% & 7.50\% \\
    SDXL + SAM3D & 0.79 & 26.71 & 11.41\% & 6.09\% & 8.75\% \\
    \midrule
    \textbf{Ours} & \textbf{0.87} & \textbf{27.23} & \textbf{73.91\%} & \textbf{63.13\%} & \textbf{68.52\%} \\
    \bottomrule
    \end{tabular}
}%
\end{table}

\begin{figure*}[tbh]
    \centering
    \includegraphics[width=1\linewidth]{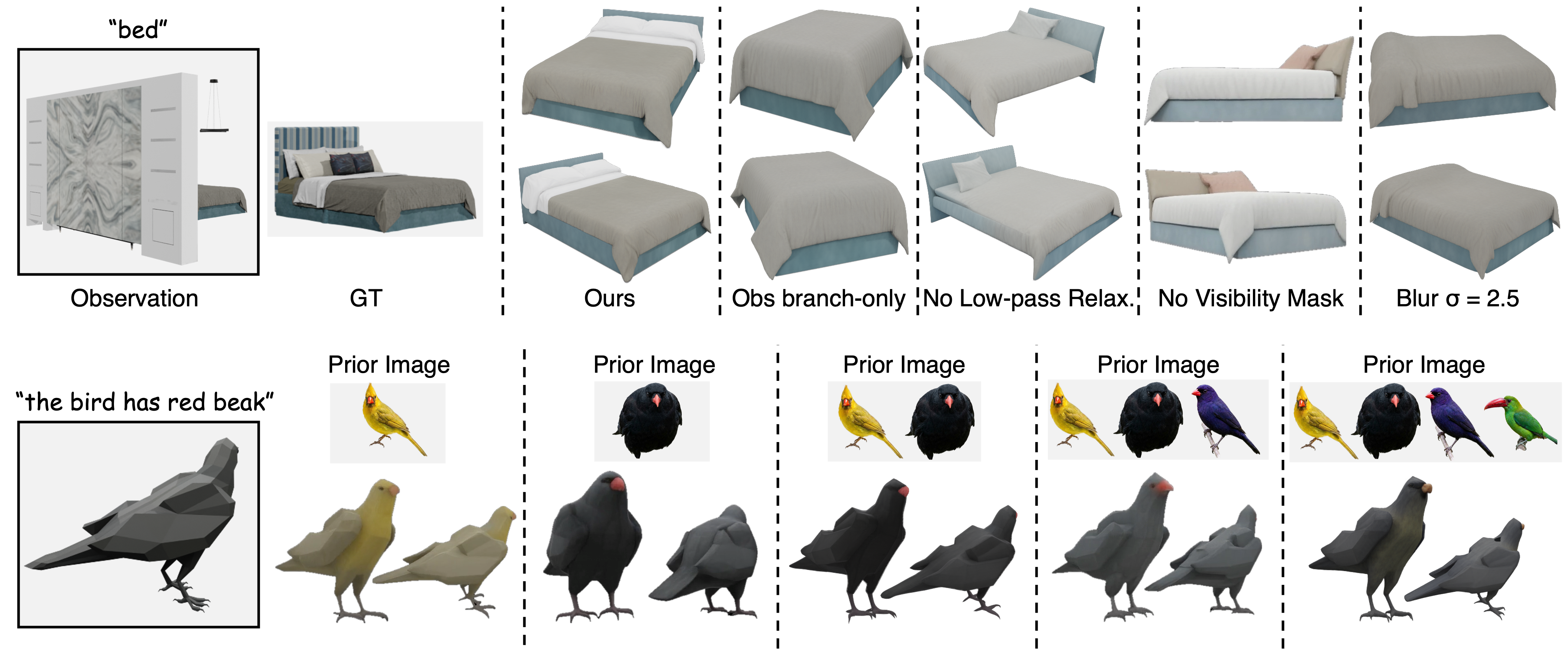}
    \caption{\small \textbf{Visual ablation results.} Top: effect of key components. Bottom: effect of varying the number of prior images $N$.}
    \label{fig:ablation_visual}
\end{figure*}

\subsection{Ablation Study}
\label{subsec:exp_ablation} 
Table~\ref{fig:ablation:a} and Figure~\ref{fig:ablation_visual} report ablations on \extremeocc (SAM3D backbone).
Removing \lprelax degrades Point-FID (81.1$\to$87.1), confirming its role in stabilizing semantic guidance.
Disabling the visibility mask causes a larger drop (81.1$\to$92.3), showing that spatially isolating occluded regions is critical.
Hyperparameter sensitivity: aggressive cutoff ($\rho$) or strong relaxation ($\sigma$) hurts performance, balancing semantic signal preservation with high-frequency suppression.

\noindent\textbf{Number of priors.}
Figure~\ref{fig:ablation_priors_metrics} shows that moderate $N$ improves consensus-based disambiguation, but excessive priors introduce conflicting details. Generated priors remain competitive (Table in Figure~\ref{fig:ablation:a}).

\noindent\textbf{Efficiency.}
Appendix~\ref{app:baseline-details} reports the runtime and memory overhead. In brief, the extra branch evaluations increase runtime, while peak memory grows modestly; the lightweight \lprelax itself is implemented with efficient 1D convolutions.

\begin{figure}[t]
    \centering
    \begin{subfigure}[t]{0.48\linewidth}
        \vspace{0pt} 
        \centering
        \scriptsize
        \begin{tabularx}{\linewidth}{@{}l|*{1}{>{\centering\arraybackslash}X}}
        \toprule
        Setting & {Point-FID$\downarrow$} \\
        \midrule
        w/o LP Relax. &     87.1 \\
        w/o Visib. Mask  &   92.3 \\
        \midrule
        cutoff $\rho=0.4$  & 86.5 \\
        cutoff $\rho=1.0$  & 89.9 \\
        LP Relax. $\sigma=2.5$  & 95.2  \\
        Prior from generation & 82.7 \\
        \midrule
        \textbf{Ours}  & \textbf{81.1} \\
        \bottomrule
        \end{tabularx}
        \subcaption{}\label{fig:ablation:a}
    \end{subfigure}
    \hfill
    \begin{subfigure}[t]{0.5\linewidth}
    \vspace{0pt} 
        \centering
        \includegraphics[width=0.97\linewidth]{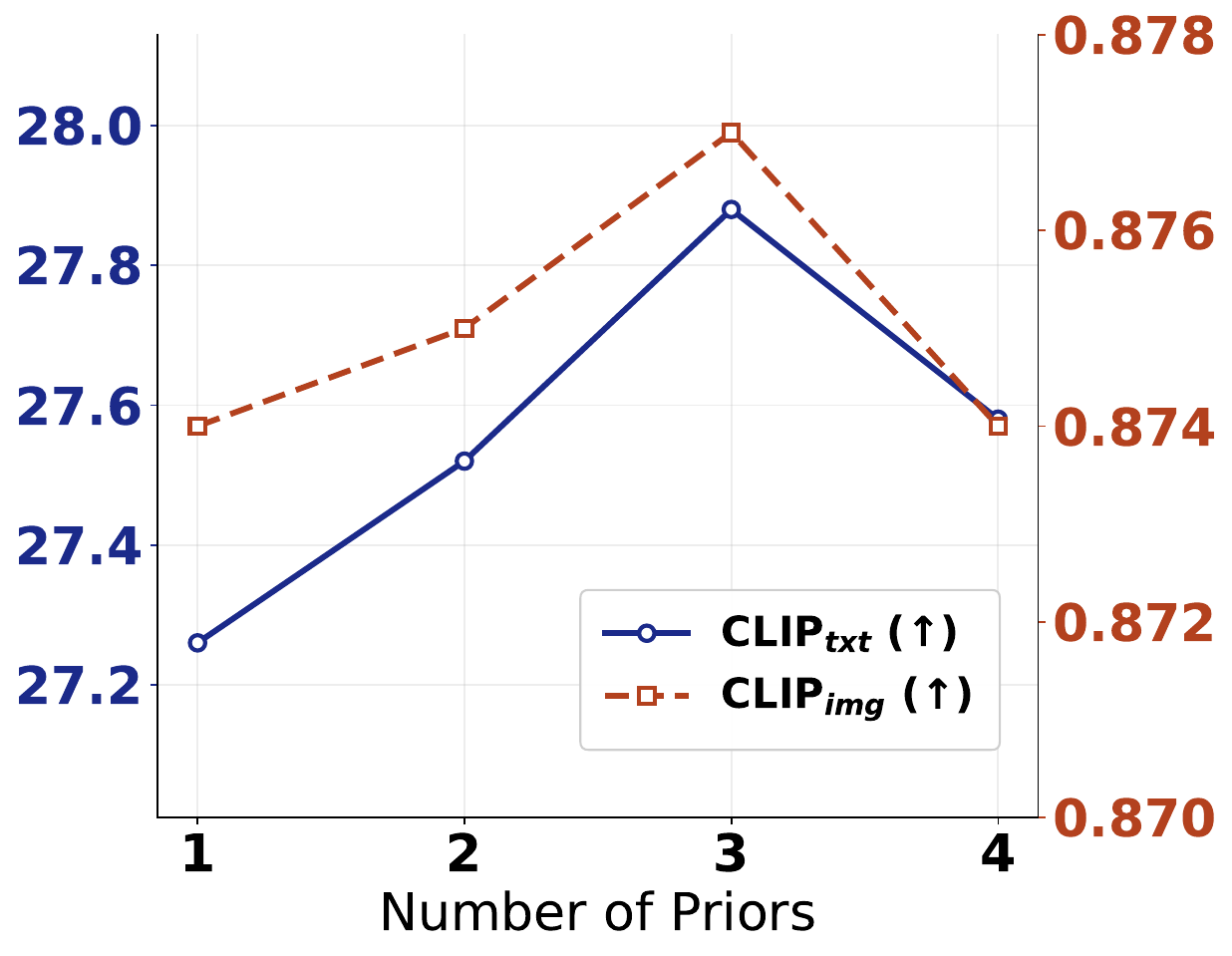}
        \vspace{-3px}
        \subcaption{}%
        \label{fig:ablation_priors_metrics}
    \end{subfigure}
    \caption{\small \textbf{Ablation studies} on \extremeocc with SAM3D backbone. (a) Component ablation and hyperparameter sensitivity. Removing \lprelax or visibility mask degrades performance; extreme $\rho$ or $\sigma$ values also hurt. ``Prior from generation'' uses Z-Image-generated priors instead of retrieved ones. (b) Effect of prior count $N$: moderate $N$ improves consensus, but too many priors introduce conflicts.}
    \label{fig:ablation}
\end{figure}

\section{Conclusion}

We introduce \task, where the same visible evidence admits multiple amodal 3D completions and users are allowed to specify intent with text prompts. To tackle this, we propose \method, a training-free dual-branch inference framework that decouples evidence preservation from semantic disambiguation: the \obsbranch anchors observation-driven high-frequency detail, while we theoretically justify the use of \lprelax in the \sembranch to distill stable, low-frequency semantic guidance from intent-conditioned priors. \methodabbr further strengthens controllability via multi-prior consensus, and a time-varying visibility-aware fusion schedule. To facilitate systematic evaluation, we introduce \extremeocc and \unseen, two diagnostic benchmarks that expose observation-overfitted collapse and evaluate intent-following under fixed evidence. Across state-of-the-art feedforward 3D generators, \methodabbr improves semantic plausibility and controllability without retraining, while preserving observation fidelity.

\section*{Acknowledgements}
We would like to acknowledge that computational work involved in this research is partially supported by NUS IT's Research Computing group using grant number NUSREC-HPC-00001. We thank the reviewers and the area chair for their constructive feedback.

\FloatBarrier

\section*{Impact Statement}
This work contributes to progress in controllable 3D generation by improving how models handle ambiguity and missing visual evidence. Such advances may support research and applications that rely on reliable 3D content, while lowering barriers to experimentation through training-free methods. At the same time, like other generative technologies, it could be misused to create misleading synthetic content or amplify biases present in underlying data and models. We encourage responsible use and deployment consistent with applicable policies and norms.

\bibliography{example_paper}

@article{team2025sam3d,
      title={SAM 3D: 3Dfy Anything in Images}, 
    author={Chen, Xingyu and Chu, Fu-Jen and Gleize, Pierre and Liang, Kevin J and Sax, Alexander and Tang, Hao and Wang, Weiyao and Guo, Michelle and Hardin, Thibaut and Li, Xiang and others},
  journal={arXiv preprint arXiv:2511.16624},
  year={2025}
}

@article{zhao2025hunyuan3d2,
  author       = {Zibo Zhao and
                  Zeqiang Lai and
                  Qingxiang Lin and
                  Yunfei Zhao and
                  Haolin Liu and
                  Shuhui Yang and
                  Yifei Feng and
                  Mingxin Yang and
                  Sheng Zhang and
                  Xianghui Yang and
                  Huiwen Shi and
                  Sicong Liu and
                  Junta Wu and
                  Yihang Lian and
                  Fan Yang and
                  Ruining Tang and
                  Zebin He and
                  Xinzhou Wang and
                  Jian Liu and
                  Xuhui Zuo and
                  Zhuo Chen and
                  Biwen Lei and
                  Haohan Weng and
                  Jing Xu and
                  Yiling Zhu and
                  Xinhai Liu and
                  Lixin Xu and
                  Changrong Hu and
                  Tianyu Huang and
                  Lifu Wang and
                  Jihong Zhang and
                  Meng Chen and
                  Liang Dong and
                  Yiwen Jia and
                  Yulin Cai and
                  Jiaao Yu and
                  Yixuan Tang and
                  Hao Zhang and
                  Zheng Ye and
                  Peng He and
                  Runzhou Wu and
                  Chao Zhang and
                  Yonghao Tan and
                  Jie Xiao and
                  Yangyu Tao and
                  Jianchen Zhu and
                  Jinbao Xue and
                  Kai Liu and
                  Chongqing Zhao and
                  Xinming Wu and
                  Zhichao Hu and
                  Lei Qin and
                  Jianbing Peng and
                  Zhan Li and
                  Minghui Chen and
                  Xipeng Zhang and
                  Lin Niu and
                  Paige Wang and
                  Yingkai Wang and
                  Haozhao Kuang and
                  Zhongyi Fan and
                  Xu Zheng and
                  Weihao Zhuang and
                  YingPing He and
                  Tian Liu and
                  Yong Yang and
                  Di Wang and
                  Yuhong Liu and
                  Jie Jiang and
                  Jingwei Huang and
                  Chunchao Guo},
  title        = {Hunyuan3D 2.0: Scaling diffusion models for high resolution textured 3D assets generation},
  journal      = {arXiv preprint arXiv:2501.12202},
  year         = {2025}
}

@article{hong2024seg,
  title={Smoothed energy guidance: Guiding diffusion models with reduced energy curvature of attention},
  author={Hong, Susung},
  journal={Advances in Neural Information Processing Systems},
  volume={37},
  pages={66743--66772},
  year={2024}
}

@inproceedings{xiang2024trellis,
    title={Structured 3d latents for scalable and versatile 3d generation},
  author={Xiang, Jianfeng and Lv, Zelong and Xu, Sicheng and Deng, Yu and Wang, Ruicheng and Zhang, Bowen and Chen, Dong and Tong, Xin and Yang, Jiaolong},
  booktitle={Proceedings of the Computer Vision and Pattern Recognition Conference},
  pages={21469--21480},
  year={2025}
}

@inproceedings{zhu2026AnchorDS,
	title = {AnchorDS: Anchoring Dynamic Sources for Semantically Consistent Text-to-3D Generation},
	author = {Zhu, Jiayin and Yang, Linlin and Li, Yicong and Yao, Angela},
	year = {2026},
	volume = {40},
	booktitle = {Proceedings of the AAAI Conference on Artificial Intelligence}
}

@article{zhu2024InstructHumans,
    author = {Zhu, Jiayin and Yang, Linlin and Yao, Angela},
    year = {2026},
    title = {InstructHumans: Editing Animated 3D Human Textures with Instructions},
    journal = {IEEE Transactions on Multimedia (TMM)}
}

@inproceedings{nipsHeuselRUNH17fid,
  title={Gans trained by a two time-scale update rule converge to a local nash equilibrium},
  author={Heusel, Martin and Ramsauer, Hubert and Unterthiner, Thomas and Nessler, Bernhard and Hochreiter, Sepp},
  booktitle={Advances in neural information processing systems},
  volume={30},
  year={2017}
}

@inproceedings{icmlRadfordKHRGASAM21clipscore,
  title={Learning transferable visual models from natural language supervision},
  author={Radford, Alec and Kim, Jong Wook and Hallacy, Chris and Ramesh, Aditya and Goh, Gabriel and Agarwal, Sandhini and Sastry, Girish and Askell, Amanda and Mishkin, Pamela and Clark, Jack and others},
  booktitle={International conference on machine learning},
  pages={8748--8763},
  year={2021}
}

@article{alex2022pointe,
  title={Point-e: A system for generating 3d point clouds from complex prompts},
  author={Nichol, Alex and Jun, Heewoo and Dhariwal, Prafulla and Mishkin, Pamela and Chen, Mark},
  journal={arXiv preprint arXiv:2212.08751},
  year={2022}
}

@InProceedings{Wu_2025_ICCV_amodal3r,
    author    = {Wu, Tianhao and Zheng, Chuanxia and Guan, Frank and Vedaldi, Andrea and Cham, Tat-Jen},
    title     = {Amodal3R: Amodal 3D Reconstruction from Occluded 2D Images},
    booktitle = {Proceedings of the IEEE/CVF International Conference on Computer Vision (ICCV)},
    month     = {October},
    year      = {2025},
    pages     = {9181-9193}
}

@inproceedings{fu_2021_iccv_3dfront,
    title={3d-front: 3d furnished rooms with layouts and semantics},
  author={Fu, Huan and Cai, Bowen and Gao, Lin and Zhang, Ling-Xiao and Wang, Jiaming and Li, Cao and Zeng, Qixun and Sun, Chengyue and Jia, Rongfei and Zhao, Binqiang and others},
  booktitle={Proceedings of the IEEE/CVF International Conference on Computer Vision},
  pages={10933--10942},
  year={2021}
}

@article{fu_2021_ijcv_3dfuture,
  title={3d-future: 3d furniture shape with texture},
  author={Fu, Huan and Jia, Rongfei and Gao, Lin and Gong, Mingming and Zhao, Binqiang and Maybank, Steve and Tao, Dacheng},
  journal={International Journal of Computer Vision},
  volume={129},
  number={12},
  pages={3313--3337},
  year={2021},
  publisher={Springer}
}

@article{qu_2025_deocc123,
  title={DeOcc-1-to-3: 3D De-Occlusion from a Single Image via Self-Supervised Multi-View Diffusion},
  author={Qu, Yansong and Dai, Shaohui and Li, Xinyang and Wang, Yuze and Shen, You and Cao, Liujuan and Ji, Rongrong},
  journal={arXiv preprint arXiv:2506.21544},
  year={2025}
}

@inproceedings{ozguroglu_2024_cvpr_pix2gestalt,
  title={pix2gestalt: Amodal segmentation by synthesizing wholes},
  author={Ozguroglu, Ege and Liu, Ruoshi and Sur{\'\i}s, D{\'\i}dac and Chen, Dian and Dave, Achal and Tokmakov, Pavel and Vondrick, Carl},
  booktitle={2024 IEEE/CVF Conference on Computer Vision and Pattern Recognition (CVPR)},
  pages={3931--3940},
  year={2024}
}

@inproceedings{ao_2025_cvpr_openWorldAmodal,
    title={Open-world amodal appearance completion},
  author={Ao, Jiayang and Jiang, Yanbei and Ke, Qiuhong and Ehinger, Krista A},
  booktitle={Proceedings of the Computer Vision and Pattern Recognition Conference},
  pages={6490--6499},
  year={2025}
}

@article{sadat_2025_guidanceInFrequency,
  title={Guidance in the Frequency Domain Enables High-Fidelity Sampling at Low CFG Scales},
  author={Sadat, Seyedmorteza and Vontobel, Tobias and Salehi, Farnood and Weber, Romann M},
  journal={arXiv preprint arXiv:2506.19713},
  year={2025}
}

@article{deitke2023objaverse,
  title={Objaverse-xl: A universe of 10m+ 3d objects},
  author={Deitke, Matt and Liu, Ruoshi and Wallingford, Matthew and Ngo, Huong and Michel, Oscar and Kusupati, Aditya and Fan, Alan and Laforte, Christian and Voleti, Vikram and Gadre, Samir Yitzhak and others},
  journal={Advances in Neural Information Processing Systems},
  volume={36},
  pages={35799--35813},
  year={2023}
}

@article{cai2025zimage,
  title={Z-Image: An Efficient Image Generation Foundation Model with Single-Stream Diffusion Transformer},
  author={Cai, Huanqia and Cao, Sihan and Du, Ruoyi and Gao, Peng and Hoi, Steven and Hou, Zhaohui and Huang, Shijie and Jiang, Dengyang and Jin, Xin and Li, Liangchen and others},
  journal={arXiv preprint arXiv:2511.22699},
  year={2025}
}

@article{ho2022cfg,
      title={Classifier-Free Diffusion Guidance}, 
      author={Jonathan Ho and Tim Salimans},
      year={2022},
        journal={arXiv preprint arXiv:2207.12598},
	publisher = {{arXiv}},
eprinttype = {arxiv},
}

@inproceedings{kulikov2025flowedit,
  title={Flowedit: Inversion-free text-based editing using pre-trained flow models},
  author={Kulikov, Vladimir and Kleiner, Matan and Huberman-Spiegelglas, Inbar and Michaeli, Tomer},
  booktitle={Proceedings of the IEEE/CVF International Conference on Computer Vision},
  pages={19721--19730},
  year={2025}
}

@InProceedings{Chen_2025_CVPR_3DTopiaXL,
    author    = {Chen, Zhaoxi and Tang, Jiaxiang and Dong, Yuhao and Cao, Ziang and Hong, Fangzhou and Lan, Yushi and Wang, Tengfei and Xie, Haozhe and Wu, Tong and Saito, Shunsuke and Pan, Liang and Lin, Dahua and Liu, Ziwei},
    title     = {3dtopia-xl: Scaling high-quality 3d asset generation via primitive diffusion},
    booktitle = {Proceedings of the IEEE/CVF Conference on Computer Vision and Pattern Recognition (CVPR)},
    month     = {June},
    year      = {2025},
    pages     = {26576-26586}
}

@article{xu2024instantmesh,
  title={InstantMesh: Efficient 3D Mesh Generation from a Single Image with Sparse-view Large Reconstruction Models},
  author={Xu, Jiale and Cheng, Weihao and Gao, Yiming and Wang, Xintao and Gao, Shenghua and Shan, Ying},
  journal={arXiv preprint arXiv:2404.07191},
  year={2024}
}

@inproceedings{lan2024ln3diff,
    title={LN3Diff: Scalable Latent Neural Fields Diffusion for Speedy 3D Generation}, 
    author={Lan, Yushi and Hong, Fangzhou and Yang, Shuai and Zhou, Shangchen and Meng, Xuyi and Dai, Bo and Pan, Xingang and Loy, Chen Change},
    year={2024},
    booktitle={ECCV},
}

@article{Liu2023Zero1to3,
  title={Zero-1-to-3: Zero-shot One Image to 3D Object},
  author={Ruoshi Liu and Rundi Wu and Basile Van Hoorick and Pavel Tokmakov and Sergey Zakharov and Carl Vondrick},
  journal={2023 IEEE/CVF International Conference on Computer Vision (ICCV)},
  year={2023},
  pages={9264-9275},
}

@inproceedings{
shi2024mvdream,
title={{MVD}ream: Multi-view Diffusion for 3D Generation},
author={Yichun Shi and Peng Wang and Jianglong Ye and Long Mai and Kejie Li and Xiao Yang},
booktitle={The Twelfth International Conference on Learning Representations},
year={2024}
}

@inproceedings{instructnerf2023,
         author = {Haque, Ayaan and Tancik, Matthew and Efros, Alexei and Holynski, Aleksander and Kanazawa, Angjoo},
         title = {Instruct-NeRF2NeRF: Editing 3D Scenes with Instructions},
         booktitle = {Proc. IEEE/CVF Int. Conf. Comput. Vis.},
         year = {2023},
}

@article{kamata_i323_2023,
	title = {Instruct 3D-to-3D: Text Instruction Guided 3D-to-3D conversion},
        year={2023},
journal={arXiv preprint arXiv:2303.15780},
	publisher = {{arXiv}},
	author = {Kamata, Hiromichi and Sakuma, Yuiko and Hayakawa, Akio and Ishii, Masato and Narihira, Takuya},
}

@inproceedings{liu2022compositionaldiffusion,
  title={Compositional visual generation with composable diffusion models},
  author={Liu, Nan and Li, Shuang and Du, Yilun and Torralba, Antonio and Tenenbaum, Joshua B},
  booktitle={Computer Vision--ECCV 2022: 17th European Conference, Tel Aviv, Israel, October 23--27, 2022, Proceedings, Part XVII},
  pages={423--439},
  year={2022},
}

@inproceedings{miao2025rethinkingsds,
title={Rethinking Score Distilling Sampling for 3D Editing and Generation},
author={Xingyu Miao and Haoran Duan and Yang Long and Jungong Han},
booktitle={Forty-second International Conference on Machine Learning},
year={2025},
}

@article{parelli20253dlatte,
  title={3D-LATTE: Latent Space 3D Editing from Textual Instructions},
  author={Parelli, Maria and Oechsle, Michael and Niemeyer, Michael and Tombari, Federico and Geiger, Andreas},
  journal={arXiv preprint arXiv:2509.00269},
  year={2025}
}

@article{podell2023sdxl,
  title={Sdxl: Improving latent diffusion models for high-resolution image synthesis},
  author={Podell, Dustin and English, Zion and Lacey, Kyle and Blattmann, Andreas and Dockhorn, Tim and M{\"u}ller, Jonas and Penna, Joe and Rombach, Robin},
  journal={arXiv preprint arXiv:2307.01952},
  year={2023}
}

@article{labs2025flux1kontext,
      title={FLUX.1 Kontext: Flow Matching for In-Context Image Generation and Editing in Latent Space},
      author={Black Forest Labs and Stephen Batifol and Andreas Blattmann and Frederic Boesel and Saksham Consul and Cyril Diagne and Tim Dockhorn and Jack English and Zion English and Patrick Esser and Sumith Kulal and Kyle Lacey and Yam Levi and Cheng Li and Dominik Lorenz and Jonas Müller and Dustin Podell and Robin Rombach and Harry Saini and Axel Sauer and Luke Smith},
      year={2025},
journal={arXiv preprint arXiv:2506.15742},
}

@article{wan2025,
  title={Wan: Open and Advanced Large-Scale Video Generative Models},
  author={Team Wan and Wang, Ang and Ai, Baole and Wen, Bin and Mao, Chaojie and Xie, Chen-Wei and Chen, Di and Yu, Feiwu and Zhao, Haiming and Yang, Jianxiao and Zeng, Jianyuan and Wang, Jiayu and Zhang, Jingfeng and Zhou, Jingren and Wang, Jinkai and Chen, Jixuan and Zhu, Kai and Zhao, Kang and Yan, Keyu and Huang, Lianghua and Feng, Mengyang and Zhang, Ningyi and Li, Pandeng and Wu, Pingyu and Chu, Ruihang and Feng, Ruili and Zhang, Shiwei and Sun, Siyang and Fang, Tao and Wang, Tianxing and Gui, Tianyi and Weng, Tingyu and Shen, Tong and Lin, Wei and Wang, Wei and Wang, Wei and Zhou, Wenmeng and Wang, Wente and Shen, Wenting and Yu, Wenyuan and Shi, Xianzhong and Huang, Xiaoming and Xu, Xin and Kou, Yan and Lv, Yangyu and Li, Yifei and Liu, Yijing and Wang, Yiming and Zhang, Yingya and Huang, Yitong and Li, Yong and Wu, You and Liu, Yu and Pan, Yulin and Zheng, Yun and Hong, Yuntao and Shi, Yupeng and Feng, Yutong and Jiang, Zeyinzi and Han, Zhen and Wu, Zhi-Fan and Liu, Ziyu},
  journal={arXiv preprint arXiv:2503.20314},
  year={2025}
}

@inproceedings{Huang_2025_ICCV,
  author={Huang, Zehuan and Guo, Yuan-Chen and Wang, Haoran and Yi, Ran and Ma, Lizhuang and Cao, Yan-Pei and Sheng, Lu},
  title={{MV-Adapter}: Multi-View Consistent Image Generation Made Easy},
  booktitle={Proceedings of the IEEE/CVF International Conference on Computer Vision (ICCV)},
  month={October},
  year={2025},
  pages={16377--16387}
}

@misc{google2025nanobananapro,
  title={Introducing Nano Banana Pro},
  author={Raisinghani, Naina},
  howpublished={Google Blog},
  year={2025},
  url={https://blog.google/technology/ai/nano-banana-pro/}
}

@inproceedings{xiaolu2026interp3d,
    title={Interp3D: Correspondence-aware Interpolation for Generative Textured 3D Morphing},
    author={Xiaolu Liu and Yicong Li and Qiyuan He and Jiayin Zhu and Wei Ji and Angela Yao and Jianke Zhu},
    booktitle={The Fourteenth International Conference on Learning Representations},
    year={2026}
}

@inproceedings{
poole2023dreamfusion,
title={DreamFusion: Text-to-3D using 2D Diffusion},
author={Ben Poole and Ajay Jain and Jonathan T. Barron and Ben Mildenhall},
booktitle={The Eleventh International Conference on Learning Representations },
year={2023},
}

@article{ye_hi3dgen_2025,
  author       = {Chongjie Ye and
                  Yushuang Wu and
                  Ziteng Lu and
                  Jiahao Chang and
                  Xiaoyang Guo and
                  Jiaqing Zhou and
                  Hao Zhao and
                  Xiaoguang Han},
  title        = {Hi3DGen: High-fidelity 3D Geometry Generation from Images via Normal
                  Bridging},
  journal      = {CoRR},
  volume       = {abs/2503.22236},
  year         = {2025}
}

@inproceedings{wu_direct3d_2024,
  author       = {Shuang Wu and
                  Youtian Lin and
                  Yifei Zeng and
                  Feihu Zhang and
                  Jingxi Xu and
                  Philip Torr and
                  Xun Cao and
                  Yao Yao},
  title        = {Direct3D: Scalable image-to-3d generation via 3d latent diffusion
                  Transformer},
  booktitle    = {NeurIPS},
  year         = {2024}
}

@inproceedings{
    huang_dreamtime_2023,
    title={DreamTime: An Improved Optimization Strategy for Diffusion-Guided 3D Generation},
    author={Yukun Huang and Jianan Wang and Yukai Shi and Boshi Tang and Xianbiao Qi and Lei Zhang},
    booktitle={Proc. Int. Conf. Learn. Represent.},
    year={2024}
}

@inproceedings{choi2022perception,
  title={Perception prioritized training of diffusion models},
  author={Choi, Jooyoung and Lee, Jungbeom and Shin, Chaehun and Kim, Sungwon and Kim, Hyunwoo and Yoon, Sungroh},
  booktitle={Proc. IEEE/CVF Conf. Comput. Vis. Pattern Recognit.},
  pages={11472--11481},
  year={2022}
}

@article{heusel2017gans,
  title={Gans trained by a two time-scale update rule converge to a local nash equilibrium},
  author={Heusel, Martin and Ramsauer, Hubert and Unterthiner, Thomas and Nessler, Bernhard and Hochreiter, Sepp},
  journal={Advances in neural information processing systems},
  volume={30},
  year={2017}
}

@article{peyre2019computational,
  title={Computational Optimal Transport},
  author={Gabriel Peyr{\'e} and Marco Cuturi},
  journal={Found. Trends Mach. Learn.},
  year={2018},
  volume={11},
  pages={355-607}
}

@article{He2023T3BenchBC,
  title={T3Bench: Benchmarking Current Progress in Text-to-3D Generation},
  author={Yuze He and Yushi Bai and Matthieu Lin and Wang Zhao and Yubin Hu and Jenny Sheng and Ran Yi and Juanzi Li and Yong-Jin Liu},
  journal={ArXiv},
  year={2023},
  volume={abs/2310.02977},
}

@inproceedings{xiu2025geometric,
  title={Geometric Alignment and Prior Modulation for View-Guided Point Cloud Completion on Unseen Categories},
  author={Xiu, Jingqiao and Li, Yicong and Zhao, Na and Fang, Han and Wang, Xiang and Yao, Angela},
  booktitle={Proceedings of the IEEE/CVF International Conference on Computer Vision},
  pages={27435--27444},
  year={2025}
}

@article{he2025rear,
  title={REAR: Rethinking Visual Autoregressive Models via Generator-Tokenizer Consistency Regularization},
  author={He, Qiyuan and Li, Yicong and Ye, Haotian and Wang, Jinghao and Liao, Xinyao and Heng, Pheng-Ann and Ermon, Stefano and Zou, James and Yao, Angela},
  journal={ICLR 2026},
  year={2025}
}

@inproceedings{li2024laso,
  title={Laso: Language-guided affordance segmentation on 3d object},
  author={Li, Yicong and Zhao, Na and Xiao, Junbin and Feng, Chun and Wang, Xiang and Chua, Tat-seng},
  booktitle={Proceedings of the IEEE/CVF Conference on Computer Vision and Pattern Recognition},
  pages={14251--14260},
  year={2024}
}

@inproceedings{li2025intermediate,
  title={Intermediate Connectors and Geometric Priors for Language-Guided Affordance Segmentation on Unseen Object Categories},
  author={Li, Yicong and Chen, Yiyang and Ma, Zhenyuan and Xiao, Junbin and Wang, Xiang and Yao, Angela},
  booktitle={Proceedings of the IEEE/CVF International Conference on Computer Vision},
  pages={22836--22845},
  year={2025}
}

@article{xiu2025egotwin,
  title={Egotwin: Dreaming body and view in first person},
  author={Xiu, Jingqiao and Hong, Fangzhou and Li, Yicong and Li, Mengze and Wang, Wentao and Han, Sirui and Pan, Liang and Liu, Ziwei},
  journal={ICLR 2026},
  year={2025}
}
\bibliographystyle{icml2026}

\newpage
\appendix
\onecolumn

\section{Theoretical Analysis of Low-Pass Relaxation via Attention-Logit Smoothing}
\label{sec:app:theory}

\subsection{Definitions and Preliminaries}

\begin{definition}[Wasserstein-$2$ distance]\label{def:W2}
    For $\mu, \nu \in \gP(\Omega)$, their Wasserstein $k$-distance is defined by
    \begin{equation*}
        \sfW_2(\mu, \nu) = \inf_{\gamma \in \Gamma(\mu, \nu)}\left(\int_{(x, y) \sim \gamma}\|x - y\|^2\od\gamma(x, y)\right)^{1/2},
    \end{equation*}
    where $\Gamma(\mu, \nu)$ is the set of all couplings of $\mu$ and $\nu$. 
\end{definition}

\subsection{Spectral analysis of semantic guidance}\label{sec:app:spectral}
We posit that global semantic intent is inherently low-frequency, whereas the errors causing artifacts are high-frequency. 
This distinction allows us to use filtering to selectively remove errors without degrading the signal. 
To justify the benefit of Gaussian blur, we characterize the spectral properties of the semantic signal versus the error. 
\begin{assumption}[Spectral concentration of semantic intent]\label[assumption]{assump:sem-tube}
    The ideal semantic transport field $\bv_{\rm sem}(\cdot, t, c_{\rm sem})$ is band-limited. In the Fourier domain, its energy is concentrated in low frequencies:
    \begin{equation}
        \bv_{\rm sem}(\omega, t, c_{\rm sem}) \approx 0 \quad \text{for } |\omega| \ge \eta,
    \end{equation}
    where $\eta > 0$ is a frequency cutoff. 
\end{assumption}
\cref{assump:sem-tube} implies $\bv_{\rm sem}$ captures global structure (low-frequency) rather than high-frequency instance details. 

For notational convenience, we write 
\begin{align}
    \epsilon(x) \coloneqq {} & \bv_{\rm sem}(x) - \bar{\bv}_{\theta}(x), \\
    \gE_{\rm obs}^2(\bv_{\theta}) \coloneqq {} &
    \int_0^t
    \Bigl\|
      \bv_{\rm obs}(X_{\tau}, \tau, c_{\rm obs}) 
      - 
      \bv_{\theta}(X_{\tau}, \tau, c_{\rm obs})
    \Bigr\|^2
    \od\tau, 
    \label{eq:def:error:det} \\
    \gE_{\rm sem}^2(\bv_{\theta}) 
    \coloneqq {} & 
    \int_0^t
    \Bigl\|
      \bv_{\rm sem}(X_{\tau}, \tau, c_{\rm sem}) 
      - 
      \bv_{\theta}(X_{\tau}, \tau, c_{\rm obs})
    \Bigr\|^2
    \od\tau. 
    \label{eq:def:error:sem}
\end{align}

\begin{assumption}[High-frequency mismatch]\label[assumption]{assump:hf-mismatch}
    The error between the learned velocity and the ideal semantic velocity is dominated by high-frequency components.  
    Let $\epsilon'$ be the Fourier transform of $\epsilon$. 
    We assume:
    \begin{equation}
        \int_{|\omega| \leq \eta} |\epsilon'(\omega)|^2 \od\omega \ll \int_{|\omega| > \eta} |\epsilon'(\omega)|^2 \od\omega.
    \end{equation}
\end{assumption}

\begin{proposition}[Error reduction via low-pass filtering]
\label{thrm:low-pass}
    Let $\tilde{\bv}_{\theta} = G_\sigma * \bar{\bv}_{\theta}$ be the velocity field smoothed by a normalized Gaussian kernel $G_\sigma$. Under \cref{assump:sem-tube,assump:hf-mismatch}, for a sufficiently small $\beta$, the semantic estimation error of the blurred field is strictly lower than that of the original field:
    \begin{equation}
        \gE_{\rm sem}(\tilde{\bv}_{\theta}) < \gE_{\rm sem}(\bar{\bv}_{\theta}). 
    \end{equation}
\end{proposition}
\begin{proof}
    Notice that
    \begin{equation*}
        \gE_{\rm sem}^2(\bar{\bv}_{\theta})
        = 
        \int_0^t
          \|\epsilon(x)\|^2
        \od\tau. 
    \end{equation*}
    Let $G_\sigma'(\omega) = \exp(-\frac{1}{2}\sigma^2\|\omega\|^2) \in (0,1], \bv_{\rm sem}'$, and $\bv_{\theta}'$ be the Fourier transform of the Gaussian kernel, $\bv_{\rm sem}$, and $\bv_{\theta}$, respectively. 
    By Parseval's theorem, the squared $L^2$ error of the smoothed field is:
    \begin{equation}
        \|\bv_{\rm sem} - \tilde{\bv}_\theta\|_{L^2}^2 = \|\bv_{\rm sem}' - G_\sigma'\bar{\bv}_\theta'\|_{L^2}^2.
    \end{equation}
    Using the fact that $\bv_{\rm sem}$ is band-limited (\cref{assump:sem-tube}) and that convolution with a normalized Gaussian preserves the low-frequency signal (assuming $G_\sigma'(\omega) \approx 1$ for $\|\omega\| \le \eta$), we focus on the error term $\epsilon$. The smoothing acts on the error as $\tilde{\epsilon} = G_\sigma * \epsilon$.
    
    We decompose the error energy into low-frequency ($\gD_{\rm low} = \{\omega : \|\omega\| \leq \eta\}$) and high-frequency ($\gD_{\rm high} = \{\omega : \|\omega\| > \eta\}$) regions:
    \begin{align*}
        \|\tilde{\epsilon}\|_{L^2}^2 &= \int_{\gD_{low}} |G_\sigma'|^2 \|\epsilon'\|^2 \od\omega + \int_{\gD_{high}} |G_\sigma'|^2 \|\epsilon'\|^2 \od\omega \\
        &\le \int_{\gD_{low}} \|\epsilon'\|^2 \od\omega + \gamma_\sigma \int_{\gD_{high}} \|\epsilon'\|^2 \od\omega,
    \end{align*}
    where $\gamma_\sigma \coloneqq \sup_{\|\omega\| > \eta} |G_\sigma'(\omega)|^2 < 1$ is the maximum gain in the high-frequency band.
    
    Comparing this to the original error $\|\epsilon\|_{L^2}^2 = \int_{\gD_{low}} \|\epsilon'\|^2 + \int_{\gD_{high}} \|\epsilon'\|^2$, the reduction is determined by the high-frequency term. Given \cref{assump:hf-mismatch}, the total error is dominated by the integral over $\gD_{high}$. Since $\gamma_\sigma < 1$, the dampening of the dominant high-frequency error strictly outweighs any potential signal loss in the low frequencies (which is negligible or zero under strict band-limiting). 
    Thus, $\|\tilde{\epsilon}\|_{L^2}^2 < \|\epsilon\|_{L^2}^2$, which implies that $\gE_{\rm sem}^2(\tilde{\bv}_{\theta}) < \gE_{\rm sem}^2(\bar{\bv}_{\theta})$.
\end{proof}

\textit{Remark:} 
In the Fourier domain, the Gaussian kernel $G_\sigma'(\omega) = e^{-\sigma^2\omega^2/2}$ acts as a low-pass filter. 
Since the semantic signal is low-frequency (\cref{assump:sem-tube}) and the error is high-frequency (\cref{assump:hf-mismatch}), filtering reduces the error term significantly while preserving the signal.

\subsection{Stability and Wasserstein bounds}\label{sec:app:stability-W2}
In this section, we translate the error reduction derived above into a guarantee on generation quality. 
We prove that by reducing high-frequency semantic noise, our method produces a distribution that is mathematically closer to the ground truth.
We define the accumulated semantic error along the trajectory as $\gE_{\rm sem}^2(\bv) = \int_0^1 \|\bv_{\rm sem}(X_\tau) - \bv(X_\tau)\|^2\od\tau$. 
We rely on the Lipschitz continuity of the neural estimator (\cref{assump:Lip:nn} in Appendix). 
\begin{assumption}[Lipschitz neural network estimator]\label[assumption]{assump:Lip:nn}
    The neural network estimator $\bar{\bv}_{\theta}(x, t, c)$ is $L_t^{\rm nn}$-Lipschitz with respect to state $x$ and condition $c$:
    \begin{align*}
    \|\bar{\bv}_{\theta}(x, t, c) - \bar{\bv}_{\theta}(x', t, c)\| &\leq L_t^{\rm nn}\|x - x'\|, \\
    \|\bar{\bv}_{\theta}(x, t, c) - \bar{\bv}_{\theta}(x, t, c')\| &\leq L_t^{\rm nn}\|c - c'\|.
    \end{align*}
\end{assumption}

\begin{proposition}[Stability of Lipschitz Constant under Gaussian Smoothing]\label{thrm:Gau:Lip}
    Let $v$ be an $L$-Lipschitz function, and let $G_\sigma$ be a normalized Gaussian kernel (i.e., $\int G_\sigma(y)\ody = 1$). 
    The convolved function $\tilde{v} \coloneqq G_\sigma * v$ is also Lipschitz, satisfying:
    \begin{equation*}
        \|\tilde{v}\|_{\mathrm{Lip}} \leq \|v\|_{\mathrm{Lip}} = L.
    \end{equation*}
\end{proposition}
\begin{proof}
For any $x, x'$, we have:
\begin{align*}
    \|\tilde{v}(x) - \tilde{v}(x')\| 
    &= \left\| \int G_\sigma(y) (v(x-y) - v(x'-y))\ody \right\| \\
    &\leq \int G_\sigma(y) \|v(x-y) - v(x'-y)\|\ody \\
    &\leq \int G_\sigma(y) L \|(x-y) - (x'-y)\|\ody \\
    &= L \|x - x'\| \underbrace{\int G_\sigma(y)\ody}_{=1} 
    = L \|x - x'\|.
\end{align*}
Thus, $\|\tilde{v}\|_{\mathrm{Lip}} \leq L$.
\end{proof}

\subsubsection{Stability Analysis}\label{sec:app:stability}

Consider the ODE flow:
\begin{align}
    \frac{\odx_t}{\odt} = {} & (1-\alpha_t)\bv_{\rm obs}(x_t, t, c_{\rm obs}) + \alpha_t\bv_{\rm sem}(x_t, t, c_{\rm sem}), 
    \quad
    x_0 \sim \rho \equiv \gN(0, \bI_d), 
    \label{eq:flow:truth} \\
    \frac{\od\bar{x}_t}{\odt} = {} & (1-\alpha_t)\bar{\bv}_{\theta}(\bar{x}_t, t, c_{\rm obs}) + \alpha_t\bar{\bv}_{\theta}(\bar{x}_t, t, c_{\rm obs}), 
    \quad
    \bar{x}_0 \sim \rho \equiv  \gN(0, \bI_d), 
    \label{eq:flow:nn} \\
    \frac{\od\hat{x}_t}{\odt} = {} & (1-\alpha_t)\bar{\bv}_{\theta}(\hat{x}_t, t, c_{\rm obs}) + \alpha_t\tilde{\bv}_{\theta}(\hat{x}_t, t, c_{\rm prior}), 
    \quad
    x_0 \sim \rho \equiv \gN(0, \bI_d).
    \label{eq:flow:Gau-blur}
\end{align}

We introduce the following definitions from optimal transport~\citep{peyre2019computational}, which are used in the proofs for our theoretical analysis: 
\begin{itemize}
    \item \textbf{Integrated ODE flow:} 
    For a velocity field $v$, the state at time $t$ is:
    \begin{equation}\label{eq:integral}
        X_t = X_0 + \int_0^t v(X_\tau, \tau)\od\tau.
    \end{equation}

    \item \textbf{Measure transportation:}
    \begin{equation}
        T(x) \coloneqq x + \int_0^T\frac{\od}{\od\tau}X_\tau\od\tau.
    \end{equation}

\item \textbf{Pushforward measure:}
    \begin{equation}
        p_0 = T_\sharp\rho,
    \end{equation}
    where $p_0$ is the underlying data distribution and $\rho$ is the reference distribution, which is generally the standard Gaussian distribution.

\item \textbf{Pullback measure:}
    \begin{equation*}
        \rho = T^\sharp p_0 = (T^{-1})_\sharp p_0.
    \end{equation*}

\end{itemize}

\begin{lemma}[Stability analysis for ODE flows]\label{thrm:ode:stability}
    Let $X_t, \overline{X}_t, \widehat{X}_t$ be the integrals~\cref{eq:integral} with the ODE flows \cref{eq:flow:truth,eq:flow:nn,eq:flow:Gau-blur}, respectively. 
    Then we have
    \begin{align*}
        \|X_t - \overline{X}_t\|
        \leq {} & 
        \Bigl(
            (1-\alpha_t)
            \gE_{\rm obs}(\bar{\bv}_{\theta}) 
            +
            \alpha_t
            \gE_{\rm sem}(\bar{\bv}_{\theta}) 
            +
            \|c_{\rm sem} - c_{\rm obs}\|
            \int_0^t
              L_{\tau}^{\rm nn}
            \od\tau
        \Bigr) 
        \exp\Bigl(
          \int_0^t
            L_{\tau}^{\rm nn}
          \od\tau
        \Bigr), 
        \\
        \|X_t - \widehat{X}_t\|
        \leq {} & 
        \Bigl(
            (1-\alpha_t)
            \gE_{\rm obs}(\bar{\bv}_{\theta}) 
            +
            \alpha_t
            \gE_{\rm sem}(\widetilde{\bv}_{\theta}) 
            +
            \|c_{\rm sem} - c_{\rm prior}\|
            \int_0^t
              \widetilde{L}_{\tau}^{\rm nn}
            \od\tau
        \Bigr)
        \exp\Bigl(
          \int_0^t
            \bigl(
            (1-\alpha_t)L_{\tau}^{\rm nn}
            +
            \alpha_t\widetilde{L}_{\tau}^{\rm nn}
            \bigr) 
          \od\tau
        \Bigr). 
    \end{align*}
\end{lemma}
\begin{proof}

By the Lipschitz assumptions for $\bv_{\rm obs}$ and $\bar{\bv}_{\theta}$, we have
\begin{align}
    \|\bv_{\rm obs}(X_t, t, c_{\rm obs}) - \bv_{\rm obs}(\overline{X}_t, t, c_{\rm obs})\| 
    \leq {} & 
    L_t^{\rm obs}
    \|X_t - \overline{X}_t\|, 
    \\
    \|\bar{\bv}_{\theta}(X_t, t, c_{\rm sem}) - \bar{\bv}_{\theta}(\overline{X}_t, t, c_{\rm sem})\| 
    \leq {} & 
    L_t^{\rm nn}
    \|X_t - \overline{X}_t\|, 
\end{align}

\textbf{(i) bounding $\|X_t - \overline{X}_t\|$}

By~\cref{eq:integral}, we have
\begin{align*}
    & \|X_t - \overline{X}_t\| \\
    = {} & 
    \Bigl\|
      x + \int_0^t\frac{\od}{\od\tau}X_{\tau}\od\tau 
      - 
      x - \int_0^t\frac{\od}{\od\tau}\overline{X}_{\tau}\od\tau
    \Bigr\| 
    \\
    \leq {} & 
    (1-\alpha_t) 
    \Bigl\|
      \int_0^t
      \left( 
        \bv_{\rm obs}(X_{\tau}, \tau, c_{\rm obs})
        -
        \bar{\bv}_{\theta}(\overline{X}_{\tau}, \tau, c_{\rm obs})
      \right)
      \od\tau
    \Bigr\| 
    +
    \alpha_t 
    \Bigl\|
      \int_0^t
      \left( 
        \bv_{\rm sem}(X_t, t, c_{\rm sem})
        -
        \bar{\bv}_{\theta}(\overline{X}_{\tau}, \tau, c_{\rm obs})
      \right)
      \od\tau
    \Bigr\| 
    \\
    \leq {} & 
    (1-\alpha_t)
    \int_0^t
      \Bigl\|
        \bv_{\rm obs}(X_{\tau}, \tau, c_{\rm obs})
        -
        \bar{\bv}_{\theta}(\overline{X}_{\tau}, \tau, c_{\rm obs})
      \Bigr\|
    \od\tau 
    + 
    \alpha_t
    \int_0^t
      \Bigl\|
        \bv_{\rm sem}(X_t, t, c_{\rm sem})
        -
        \bar{\bv}_{\theta}(\overline{X}_{\tau}, \tau, c_{\rm obs})
      \Bigr\|
    \od\tau. 
\end{align*}

For the first term:
\begin{align}
     {} &
     \int_0^t
      \Bigl\|
        \bv_{\rm obs}(X_{\tau}, \tau, c_{\rm obs})
        -
        \bar{\bv}_{\theta}(\overline{X}_{\tau}, \tau, c_{\rm obs})
      \Bigr\|
    \od\tau 
    \notag \\
    \leq {} &
    \int_0^t
      \Bigl\|
        \bv_{\rm obs}(X_{\tau}, \tau, c_{\rm obs})
        -
        \bar{\bv}_{\theta}(X_{\tau}, \tau, c_{\rm obs})
      \Bigr\|
    \od\tau 
    + 
     \int_0^t
      \Bigl\|
        \bar{\bv}_{\theta}(X_{\tau}, \tau, c_{\rm obs})
        -
        \bar{\bv}_{\theta}(\overline{X}_{\tau}, \tau, c_{\rm obs})
      \Bigr\|
    \od\tau 
    \notag \\
    \leq {} &
    \Bigl( 
    \int_0^t
      \Bigl\|
        \bv_{\rm obs}(X_{\tau}, \tau, c_{\rm obs})
        -
        \bar{\bv}_{\theta}(X_{\tau}, \tau, c_{\rm obs})
      \Bigr\|^2
    \od\tau 
    \Bigr)^{1/2}
    +
    \int_0^t
      L_{\tau}^{\rm nn}
      \|X_{\tau} - \overline{X}_{\tau}\|
    \od\tau
    \tag{by Cauchy-Schwarz inequality and Lipschitz of $\bar{\bv}_{\theta}$} \\
    = {} &
    \gE_{\rm obs}(\bar{\bv}_{\theta}) 
    +
    \int_0^t
      L_{\tau}^{\rm nn}
      \|X_{\tau} - \overline{X}_{\tau}\|
    \od\tau. 
    \tag{by the notation defined in \cref{eq:def:error:det}}
\end{align}

Similarly, for the second term, we have
\begin{align}
    {} &
    \int_0^t
      \Bigl\|
        \bv_{\rm sem}(X_t, t, c_{\rm sem})
        -
        \bar{\bv}_{\theta}(\overline{X}_{\tau}, \tau, c_{\rm obs})
      \Bigr\|
    \od\tau
    \notag \\
    \leq {} &
    \int_0^t
      \Bigl\|
        \bv_{\rm sem}(X_{\tau}, \tau, c_{\rm sem})
        -
        \bar{\bv}_{\theta}(X_{\tau}, \tau, c_{\rm sem})
      \Bigr\|
    \od\tau 
    + 
     \int_0^t
      \Bigl\|
        \bar{\bv}_{\theta}(X_{\tau}, \tau, c_{\rm sem})
        -
        \bar{\bv}_{\theta}(\overline{X}_{\tau}, \tau, c_{\rm obs})
      \Bigr\|
    \od\tau 
    \notag \\
    \leq {} &
    \Bigl( 
    \int_0^t
      \Bigl\|
        \bv_{\rm sem}(X_{\tau}, \tau, c_{\rm obs})
        -
        \bar{\bv}_{\theta}(X_{\tau}, \tau, c_{\rm obs})
      \Bigr\|^2
    \od\tau 
    \Bigr)^{1/2}
    +
    \|c_{\rm sem} - c_{\rm obs}\|
    \int_0^t
      L_{\tau}^{\rm nn}
    \od\tau
    +
    \int_0^t
      L_{\tau}^{\rm nn}
      \|X_{\tau} - \overline{X}_{\tau}\|
    \od\tau
    \tag{by Cauchy-Schwarz inequality and Lipschitz of $\bar{\bv}_{\theta}$} \\
    = {} &
    \gE_{\rm sem}(\bar{\bv}_{\theta}) 
    +
    \|c_{\rm sem} - c_{\rm obs}\|
    \int_0^t
      L_{\tau}^{\rm nn}
    \od\tau
    +
    \int_0^t
      L_{\tau}^{\rm nn}
      \|X_{\tau} - \overline{X}_{\tau}\|
    \od\tau. 
    \tag{by the notation defined in \cref{eq:def:error:sem}}
\end{align}

Combining the above two inequalities, we obtain
\begin{align*}
    \Bigl\|X_t - \overline{X}_t\Bigr\|
    \leq 
    (1-\alpha_t)
    \gE_{\rm obs}(\bar{\bv}_{\theta}) 
    +
    \alpha_t
    \gE_{\rm sem}(\bar{\bv}_{\theta}) 
    +
    \|c_{\rm sem} - c_{\rm obs}\|
    \int_0^t
      L_{\tau}^{\rm nn}
    \od\tau
    +
    \int_0^t
      L_{\tau}^{\rm nn}
      \|X_{\tau} - \overline{X}_{\tau}\|
    \od\tau. 
\end{align*}

Using Gr\"onwall's inequality (see~\cref{lem:Gronwall-ineq}), we get
\begin{align*}
    \|X_t - \overline{X}_t\|
    \leq 
    \Bigl(
        (1-\alpha_t)
        \gE_{\rm obs}(\bar{\bv}_{\theta}) 
        +
        \alpha_t
        \gE_{\rm sem}(\bar{\bv}_{\theta}) 
        +
        \|c_{\rm sem} - c_{\rm obs}\|
        \int_0^t
          L_{\tau}^{\rm nn}
        \od\tau
    \Bigr) 
    \exp\Bigl(
      \int_0^t
        L_{\tau}^{\rm nn}
      \od\tau
    \Bigr). 
\end{align*}

\textbf{(ii) bounding $\|X_t - \widehat{X}_t\|$}

Following the same derivations in (i), we have
\begin{align*}
    \|X_t - \widehat{X}_t\|
    \leq {} & 
    (1-\alpha_t)
    \int_0^t
      \Bigl\|
        \bv_{\rm obs}(X_{\tau}, \tau, c_{\rm obs})
        -
        \bar{\bv}_{\theta}(\widehat{X}_{\tau}, \tau, c_{\rm obs})
      \Bigr\|
    \od\tau 
    + 
    \alpha_t
    \int_0^t
      \Bigl\|
        \bv_{\rm sem}(X_t, t, c_{\rm sem})
        -
        \tilde{\bv}_{\theta}(\widehat{X}_{\tau}, \tau, c_{\rm obs})
      \Bigr\|
    \od\tau. 
\end{align*}

The first term can be upper-bounded as in (i). 
For the second term, following a similar derivation as in (i) for the second term, we obtain
\begin{align*}
    {} & 
    \int_0^t
      \Bigl\|
        \bv_{\rm sem}(X_t, t, c_{\rm sem})
        -
        \tilde{\bv}_{\theta}(\widehat{X}_{\tau}, \tau, c_{\rm obs})
      \Bigr\|
    \od\tau
    \\
    \leq {} &
    \gE_{\rm sem}(\tilde{\bv}_{\theta}) 
    +
    \|c_{\rm sem} - c_{\rm prior}\|
    \int_0^t
        \widetilde{L}_{\tau}^{\rm nn}
    \od\tau
    +
    \int_0^t
      \widetilde{L}_{\tau}^{\rm nn}
      \|X_{\tau} - \overline{X}_{\tau}\|
    \od\tau. 
\end{align*}

Combining these two inequalities, we obtain
\begin{equation*}
    \|X_t - \widehat{X}_t\|
    \leq 
    (1-\alpha_t)
    \gE_{\rm obs}(\bar{\bv}_{\theta}) 
    +
    \alpha_t
    \gE_{\rm sem}(\widetilde{\bv}_{\theta}) 
    +
    \|c_{\rm sem} - c_{\rm prior}\|
    \int_0^t
      \widetilde{L}_{\tau}^{\rm nn}
    \od\tau
    +
    \int_0^t
      \Bigl(
        (1-\alpha_t)L_{\tau}^{\rm nn}
        +
        \alpha_t\widetilde{L}_{\tau}^{\rm nn}
      \Bigr) 
      \|X_{\tau} - \overline{X}_{\tau}\|
    \od\tau. 
\end{equation*}

Again, applying Gr\"onwall's inequality, we get
\begin{align*}
    \|X_t - \widehat{X}_t\|
    \leq 
    \Bigl(
        (1-\alpha_t)
        \gE_{\rm obs}(\bar{\bv}_{\theta}) 
        +
        \alpha_t
        \gE_{\rm sem}(\widetilde{\bv}_{\theta}) 
        +
        \|c_{\rm sem} - c_{\rm prior}\|
        \int_0^t
          \widetilde{L}_{\tau}^{\rm nn}
        \od\tau
    \Bigr)
    \exp\Bigl(
      \int_0^t
        \bigl(
        (1-\alpha_t)L_{\tau}^{\rm nn}
        +
        \alpha_t\widetilde{L}_{\tau}^{\rm nn}
        \bigr) 
      \od\tau
    \Bigr). 
\end{align*}

\end{proof}

\subsubsection{Wasserstein Bounds}\label{sec:app:W2-bound}

\begin{theorem}\label{thrm:Wasserstein-bound}
    Let $T, \overline{T}, \widehat{T}$ be the transport maps induced by the ground truth flow~\cref{eq:flow:truth}, the neural network flow~\cref{eq:flow:nn}, and the Gaussian blur flow~\cref{eq:flow:Gau-blur}, respectively. Let $\rho = \gN(0, \bI_d)$.
    Let $\rho = \gN(0, \bI_d)$.
    Then the Wasserstein-2 distance between the generated distributions is bounded by:
    \begin{align*}
        \sfW_2(T_\sharp\rho, \overline{T}_\sharp\rho) 
        \leq {} & 
        \Bigl(
            (1-\alpha_{\max}) \gE_{\rm obs}(\bar{\bv}_{\theta}) 
            + \alpha_{\max} \gE_{\rm sem}(\bar{\bv}_{\theta}) 
            + C_{\rm obs}
        \Bigr)
        \exp\Bigl( \int_0^t L_{\tau}^{\rm nn} \od\tau \Bigr),
        \\
        \sfW_2(T_\sharp\rho, \widehat{T}_\sharp\rho) 
        \leq {} & 
        \Bigl(
            (1-\alpha_{\max})
            \gE_{\rm obs}(\bar{\bv}_{\theta}) 
            +
            \alpha_{\max} \gE_{\rm sem}(\widetilde{\bv}_{\theta}) 
            +
            C_{\rm prior}
        \Bigr)
        \exp\Bigl(
          \int_0^t
            \bigl(
            (1-\alpha_t)L_{\tau}^{\rm nn}
            +
            \alpha_t\widetilde{L}_{\tau}^{\rm nn}
            \bigr) 
          \od\tau
        \Bigr). 
    \end{align*}
    where $C_{\rm obs} = \|c_{\rm sem} - c_{\rm obs}\| \int_0^t L_{\tau}^{\rm nn} \od\tau$, and $C_{\rm prior} = \|c_{\rm sem} - c_{\rm prior}\| \int_0^t \widetilde{L}_{\tau}^{\rm nn} \od\tau$.
\end{theorem}

\begin{proof}
    By the definition of the Wasserstein distance, we have:
    \begin{align*}
        \sfW_2^2(T_\sharp\rho, \overline{T}_\sharp\rho) 
        = {} & \inf_{\gamma \in \Gamma(T_\sharp\rho, \overline{T}_\sharp\rho)} \int \|x - y\|^2 \od\gamma(x, y).
    \end{align*}
    
    We construct a specific coupling $\gamma^*$ induced by the shared initialization. Let $z \sim \rho$ be the common latent variable. Let $x = T(z) = X_t$ and $y = \overline{T}(z) = \overline{X}_t$. The joint distribution of $(X_t, \overline{X}_t)$ forms a valid coupling in $\Gamma(T_\sharp\rho, \overline{T}_\sharp\rho)$.
    
    Thus,
    \begin{align*}
        \sfW_2^2(T_\sharp\rho, \overline{T}_\sharp\rho) 
        \leq {} & \E_{z \sim \rho} \Bigl[ \|T(z) - \overline{T}(z)\|^2 \Bigr] \\
        = {} & \E_{z \sim \rho} \Bigl[ \|X_t - \overline{X}_t\|^2 \Bigr].
    \end{align*}
    
    Using the Stability Lemma derived in \cref{thrm:ode:stability}, we substitute the bound for $\|X_t - \overline{X}_t\|$. Let $K_t = \exp(\int_0^t L_{\tau}^{\rm nn} \od\tau)$.
    
    \begin{align*}
        \E \Bigl[ \|X_t - \overline{X}_t\|^2 \Bigr]
        \leq {} & K_t^2 \cdot \E \left[ \left( 
            (1-\alpha_t) \int_0^t \|\Delta \bv_{\rm obs}\| \od\tau 
            + \alpha_t \int_0^t \|\Delta \bv_{\rm sem}\| \od\tau 
            + C_{\rm sem}
        \right)^2 \right].
    \end{align*}
    
    Applying the Minkowski inequality (triangle inequality for $L_2$ norm) to the expectation terms yields the final result in terms of $\gE_{\rm obs}$ and $\gE_{\rm sem}$:
    \begin{equation*}
        \left(\E [\|X_t - \overline{X}_t\|^2]\right)^{1/2} 
        \leq 
        K_t \left( 
            (1-\alpha_t) \gE_{\rm obs}(\bar{\bv}_{\theta}) 
            + \alpha_t \gE_{\rm sem}(\bar{\bv}_{\theta}) 
            + C_{\rm sem}
        \right), 
    \end{equation*}
    which gives 
    \begin{align*}
        \sfW_2(T_\sharp\rho, \overline{T}_\sharp\rho) 
        \leq {} &
        \Bigl(
          \E_{z \sim \rho} \Bigl[ \|X_t - \overline{X}_t\|^2 \Bigr]
        \Bigr)^{1/2}
        \\
        \leq {} &
        \left( 
            (1-\alpha_t) \gE_{\rm obs}(\bar{\bv}_{\theta}) 
            + \alpha_t \gE_{\rm sem}(\bar{\bv}_{\theta}) 
            + C_{\rm sem}
        \right)
        \exp(\int_0^t L_{\tau}^{\rm nn} \od\tau). 
    \end{align*}

    Follow the similar proof for $\sfW_2(T_\sharp\rho, \overline{T}_\sharp\rho)$, we can bound $\sfW_2(T_\sharp\rho, \widehat{T}_\sharp\rho)$.  
\end{proof}

\begin{theorem}[Wasserstein Distance Bound]\label{thrm:Wasserstein-bound:main}
    Let $p$ be the distribution generated by the Ground Truth Flow, $\bar{p}$ by the Standard Flow, and $\hat{p}$ by the \methodabbr. 
    The Wasserstein-2 distance to the ground truth is bounded by:
    \begin{align}
        \gW_2(p, \bar{p}) &\leq C \cdot \left(\gE_{\rm obs} + \gE_{\rm sem}(\bar{\bv}_{\theta}) \right), \\
        \gW_2(p, \hat{p}) &\leq C \cdot \left(\gE_{\rm obs} + \gE_{\rm sem}(\tilde{\bv}_{\theta})+ \delta_{\rm prior} \right),
    \end{align}
    where $C$ is a constant depending on the Lipschitz continuity of the network (Grönwall factor), and $\delta_{\rm prior}$ accounts for the mismatch between $c_{\rm sem}$ and $c_{\rm prior}$.
\end{theorem}
\begin{proof}
    Let $\bar{p} = \overline{T}_\sharp \rho, \widehat{p} = \widehat{T}_\sharp \rho$ be the distribution generated by neural network flow~\cref{eq:flow:nn} and the Gaussian blur flow~\cref{eq:flow:Gau-blur}, respectively. 
    The proof follows identically to \cref{thrm:Wasserstein-bound} by choosing the coupling $(X_t, \widehat{X}_t)$ induced by $z \sim \rho$ and applying the second inequality from the Stability Lemma.
\end{proof}
Since $\gE_{\rm sem}(\tilde{\bv}_{\theta}) < \gE_{\rm sem}(\bar{\bv}_{\theta})$ (due to the filtering of high-frequency noise), the Gaussian blur flow $\hat{p}$ achieves a tighter Wasserstein bound to the ground truth distribution $p$ than the standard flow $\bar{p}$, provided the conditioning mismatch $\delta_{\rm prior}$ is controlled.

\begin{lemma}[Gr\"onwall's Inequality]\label{lem:Gronwall-ineq}
    Assume that the continuous functions $u, \kappa: [0, T] \rightarrow [0, \infty)$ and constant $K > 0$ satisfy:
    \begin{equation*}
        u(t) \leq K + \int_0^t \kappa(s)u(s)\mathrm{d}s,
    \end{equation*}
    for all $t \in [0, T]$. Then:
    \begin{equation*}
        u(t) \leq K\exp\left(\int_0^t \kappa(s)\mathrm{d}s\right).
    \end{equation*}
\end{lemma}

\subsection{Mechanism: Logit Smoothing as Vector Field Relaxation}
\label{appen:logit-smooth-as-vector-relax}

Section~\ref{sec:lowpass_theory} analyzes low-pass relaxation abstractly as an operator $\mathcal{R}_\sigma$ acting on the prior-conditioned velocity field (Eq.~\ref{eq:lp_relax_def}).
In implementation, we realize this relaxation by blurring the prior-branch cross-attention logits before the softmax, which we now connect to an induced relaxation on the resulting velocity.

Let $v_\theta(x,t,c)$ denote the rectified-flow velocity produced by the transformer, and let $L(x,t,c)$ denote the corresponding cross-attention logits inside the network.
In the \sembranch, we replace $L$ by $\tilde L = G_\sigma * L$ (Eq.~\ref{eq:attn_blur}) and denote the induced velocity by $\tilde v_\theta(x,t,c_{\rm prior}) := v_\theta(x,t,c_{\rm prior};\tilde L)$.

\begin{assumption}[Logit-to-velocity smoothness]\label{assump:logit-lip}
Assume the velocity is Lipschitz in the logits with constant $C_t$:
\[
\|v_\theta(x,t,c;L_1)-v_\theta(x,t,c;L_2)\| \le C_t \|L_1-L_2\|.
\]
\end{assumption}

Under Assumption~\ref{assump:logit-lip} and the fact that Gaussian blurring contracts high-frequency logit components, we obtain
\[
\|\tilde v_\theta(x,t,c_{\rm prior}) - v_\theta(x,t,c_{\rm prior})\|
\;\le\;
C_t \,\|\tilde L - L\|,
\]
showing that logit smoothing induces a controlled relaxation of the prior-conditioned velocity.

\section{Dataset Curation}\label{app:dataset-curation}

\subsection{\extremeocc Dataset}

\paragraph{Data source.}
We build \extremeocc from 3D-FUTURE~\cite{fu_2021_ijcv_3dfuture} and 3D-FRONT~\cite{fu_2021_iccv_3dfront}, which provide indoor scenes with realistic furniture arrangements and high-quality ground-truth 3D objects. The rendered images capture natural viewpoints where occlusions and ambiguity arise organically.

\paragraph{Occlusion ratio computation.}
Using ground-truth meshes and poses, we reproject each object's amodal silhouette onto the scene image and compare it with the visible mask to compute the occlusion ratio $r$. Empirically, when $r \ge 80\%$, humans struggle to identify the object type from visible evidence alone.

\paragraph{Failure modes under extreme occlusion.}
State-of-the-art occlusion-aware models (\eg, SAM3D~\cite{team2025sam3d}) typically fail under such conditions in two ways: (1) generating incomplete, implausible geometry, or (2) producing a semantically incorrect but visually plausible object (\eg, generating a chair for an occluded bed in a bedroom scene).

\paragraph{Final dataset.}
We filter approximately 2000 candidates with $r \ge 80\%$ and obtain $264$ test cases. Each case is paired with a category-level text prior. %
Ground-truth 3D meshes are provided for evaluation. 
Fig.~\ref{fig:cases-extreme} shows examples of the \extremeocc dataset.

\subsection{\unseen Dataset} 

\paragraph{Motivation.}
In many real-world 3D generation settings—such as single-view reconstruction, novel view synthesis, or partial observation—the input image does not uniquely determine the underlying 3D structure or identity.
Humans naturally rely on language and prior knowledge to resolve such ambiguity. However, existing benchmarks rarely test whether text guidance actually resolves semantic uncertainty, rather than merely refining appearance or style.

\unseen is a diagnostic benchmark designed to evaluate text-guided 3D generation under incomplete or ambiguous visual observation.
Unlike conventional 3D datasets that assume a single correct interpretation for each input, \unseen focuses on semantic branching scenarios, where a single visual input admits multiple plausible but mutually exclusive semantic explanations.

\paragraph{Curation process.}
Three annotators manually reviewed rendered images from a high-quality diverse 3D object dataset, ObjaverseXL~\cite{deitke2023objaverse}, and carefully designed $21$ cases exhibiting one of three ambiguity types:
\begin{itemize}%
    \item \textbf{Masked semantic completion:} key semantic regions are occluded.
    \item \textbf{View-induced identity ambiguity:} limited viewpoints prevent unique identification.
    \item \textbf{Intrinsic shape ambiguity:} geometry admits multiple semantic interpretations.
\end{itemize}

\paragraph{Distinction from training data.}
These cases differ fundamentally from typical training data: we use single rendered images from ambiguous viewpoints or under heavy occlusion, paired with text priors leading to diverse target objects distinct from the original 3D assets. This tests whether models can overcome their learned biases and generate diverse, branch-consistent outputs.

\paragraph{Dataset structure.}
Each case consists of: (1) a single ambiguous input image, and (2) a set of text branches, where each branch represents a distinct but plausible semantic interpretation. All branches are valid given the visual evidence, yet correspond to different 3D structures or categories. 
Fig.~\ref{fig:cases-amb} shows all samples of the \unseen dataset. Qualitative results are shown in Fig.~\ref{fig:qual-all}, Fig.~\ref{fig:appen-more-qual-shield}, Fig.~\ref{fig:appen-more-qual-backboard}, and Fig.~\ref{fig:appen-more-qual-3}.

\begin{figure}
    \centering
    \includegraphics[width=1\linewidth]{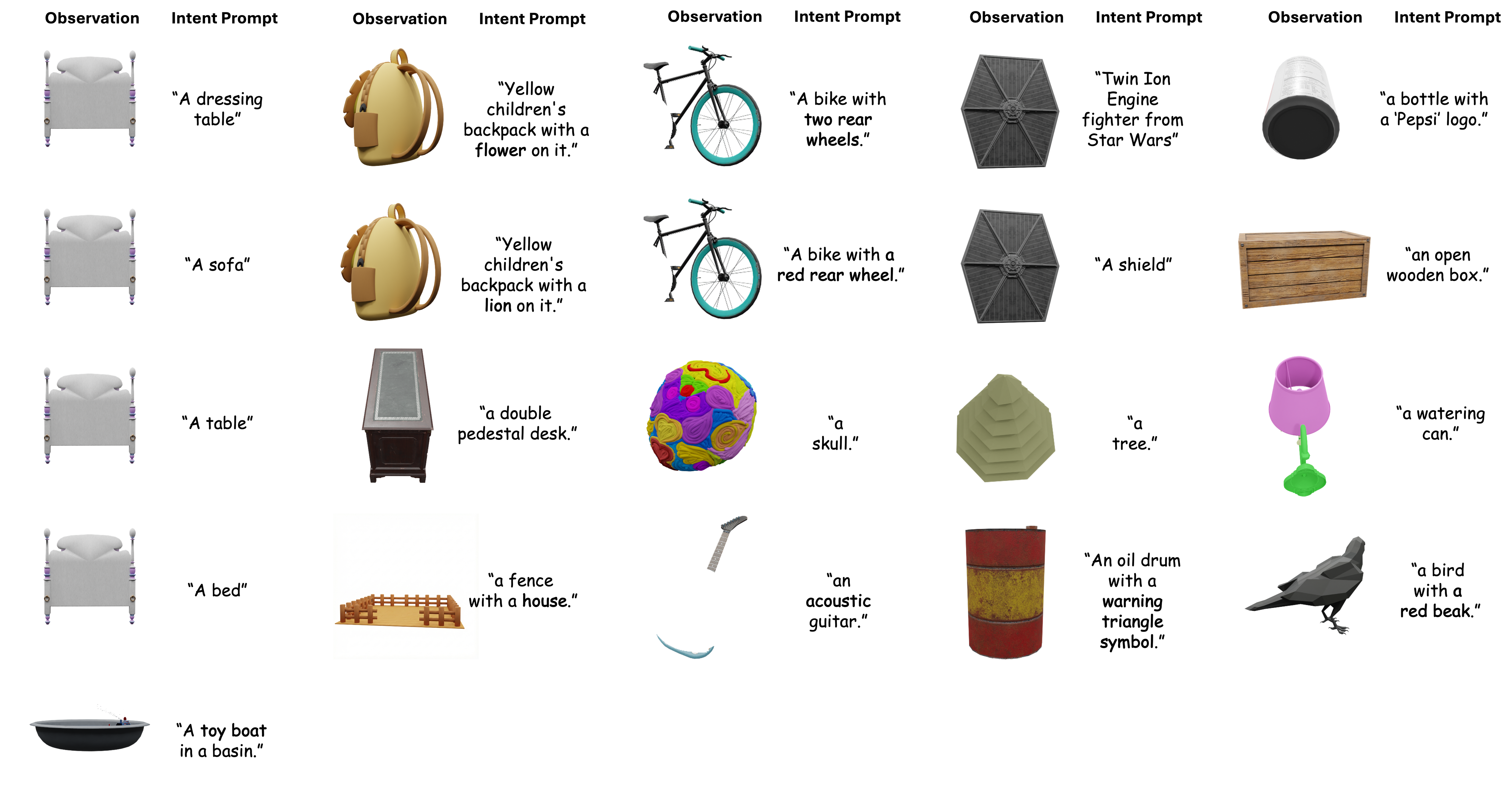}
    \caption{Visualization of all cases in AmbiSem-3D.}
    \label{fig:cases-amb}
\end{figure}

\begin{figure}
    \centering
    \includegraphics[width=1\linewidth]{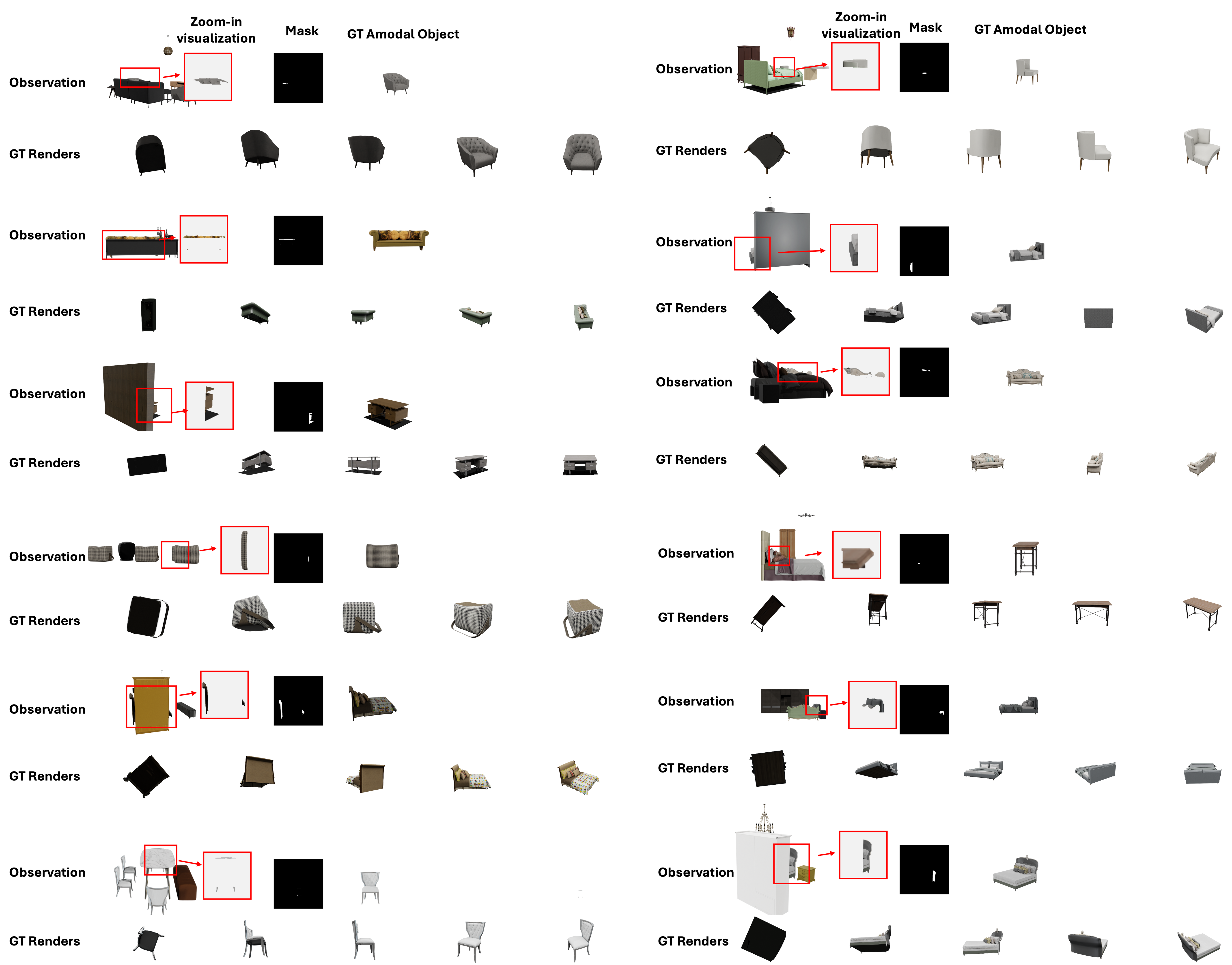}
    \caption{More examples of cases in ExtremeOcc-3D.}
    \label{fig:cases-extreme}
\end{figure}

\paragraph{100-case extension.}
For the extended evaluation in Appendix~\ref{app:baseline-details}, we construct a larger \unseen-Ext split using a semi-automatic candidate mining pipeline. We first render ObjaverseXL assets from multiple viewpoints and caption each view with a VLM. A second VLM/LLM agent then compares the captions across views and identifies cases where one view yields a substantially different semantic interpretation from the consensus of the remaining views. Such views are treated as candidate ambiguous observations, since their visible evidence may support multiple plausible completions. The automatic stage mines 120 candidate cases, which we then manually filter to remove 20 clearly unreasonable samples, such as cases with insufficient asset quality or implausible ambiguity. The resulting extension contains 100 retained cases and is intended to broaden the evaluation toward view-induced ambiguity, complementing the inherent ambiguities emphasized by the 21-case diagnostic set.

\section{Prior Image Construction}\label{app:prior-construction}

Our method requires prior images to provide semantic guidance. We describe how priors are obtained for each benchmark during our experiments.

\paragraph{\extremeocc.}
Given a text label specifying the target category, we retrieve a prior image by randomly sampling a \emph{different} object of the same category from the 3D-FUTURE dataset and rendering it from a canonical viewpoint. This retrieval-based approach ensures that the prior shares the intended category but differs in instance-specific details, testing whether our method extracts only category-level semantics.

\paragraph{\unseen.}
Since \unseen involves diverse semantic branches without a retrieval pool, we generate prior images from text prompts using a text-to-image generator Z-Image~\cite{cai2025zimage}. To ensure clean, 3D-compatible outputs, we augment the user prompt with a suffix: ``\texttt{, clean 3D style, pure white background, single object in the center}.'' For fair comparison, all baseline methods on \unseen receive the same generated prior images as input.

\paragraph{Ablation on prior source.}
To verify robustness to prior source, we also evaluate \methodabbr on \extremeocc using Z-Image-generated priors~\cite{cai2025zimage} instead of retrieved images. As shown in Table~\ref{fig:ablation:a} (``Prior from generation''), performance remains competitive (Point-FID: 82.7 vs.\ 81.1), indicating that our method is not sensitive to whether priors are retrieved or generated.

\section{Details of Metric Computation}
\label{app:metric-details}

We report a set of complementary metrics that capture semantic alignment to the text prior, semantic/appearance agreement with reference renders, fidelity at the distribution level, and 3D semantic similarity. For each test instance, our method produces $V{=}10$ predicted views $\mathcal{V}_{\text{pred}}=\{v_i\}_{i=1}^{V}$, together with a prior text prompt $t$. We also consider a set of processed ground-truth multi-view renders $\mathcal{V}_{\text{gt}}^{\text{proc}}$ and an observed object render $o^{\text{proc}}$ (from a specific viewpoint). When a metric requires standardized inputs, we apply a consistent preprocessing pipeline (center-crop and resize), yielding processed predicted views $\mathcal{V}_{\text{pred}}^{\text{proc}}=\{\tilde v_i\}_{i=1}^{V}$.

\paragraph{CLIP image--text similarity.}
We measure how well the predicted views align with the textual prior using CLIP image--text cosine similarity. Let $f_I(\cdot)$ and $f_T(\cdot)$ denote the CLIP image and text encoders, and $\hat{x}=x/\|x\|_2$ the $\ell_2$-normalized embedding. We compute
\begin{equation}
s_{\text{IT}}
=\frac{1}{V}\sum_{i=1}^{V}
\left\langle \widehat{f_I(v_i)}, \widehat{f_T(t)} \right\rangle ,
\end{equation}
where higher values indicate stronger semantic agreement between the generated imagery and the text prompt.

\paragraph{CLIP image--image similarity (multi-view semantic agreement).}
To quantify overall agreement between the predicted object and the reference object at the semantic level, we compare the predicted views against a set of reference images that aggregates ground-truth renders and the observed render:
\begin{equation}
\mathcal{R} = \mathcal{V}_{\text{gt}}^{\text{proc}} \cup \{o^{\text{proc}}\}.
\end{equation}
We embed each image with the CLIP image encoder, form the mean reference embedding, and average cosine similarities over predicted views:
\begin{equation}
\mu_{\mathcal{R}}
=\frac{1}{|\mathcal{R}|}\sum_{r\in\mathcal{R}} \widehat{f_I(r)},
\qquad
s_{\text{II}}
=\frac{1}{V}\sum_{i=1}^{V}
\left\langle \widehat{f_I(\tilde v_i)}, \widehat{\mu_{\mathcal{R}}} \right\rangle .
\end{equation}
This metric summarizes \emph{multi-view} consistency at a semantic representation level, reflecting whether the predicted set depicts the same underlying object content as the reference set (higher is better).

\paragraph{Minimum LPIPS distance (single-view appearance consistency).}
While CLIP image--image similarity emphasizes \emph{set-level, multi-view} semantic agreement, we additionally evaluate how well the method preserves the \emph{observed} appearance from its specific viewpoint using LPIPS. Let $d_{\text{LPIPS}}(\cdot,\cdot)$ denote the LPIPS perceptual distance. We compute the minimum distance between the observed render and the predicted views:
\begin{equation}
d_{\min}
=
\min_{i\in\{1,\dots,V\}} d_{\text{LPIPS}}\!\left(o^{\text{proc}},\, \tilde v_i\right).
\end{equation}
Lower $d_{\min}$ indicates that at least one predicted view closely matches the observed image in perceptual appearance, which is particularly relevant because the observation provides supervision only at that single viewpoint.

\paragraph{FID on multi-view image sets.}
We also report Fr\'echet Inception Distance (FID)~\cite{nipsHeuselRUNH17fid} to assess distribution-level realism and coverage of the predicted multi-view images relative to ground-truth renders. For each object, we treat its $V$ predicted views as samples contributing to an overall generated set, and analogously aggregate ground-truth views into a reference set. Using Inception features $\phi(\cdot)$, we fit Gaussians to the feature distributions of the generated and reference sets, with means and covariances $(m_g, C_g)$ and $(m_r, C_r)$, respectively. The FID is
\begin{equation}
\mathrm{FID}
=
\|m_g-m_r\|_2^2
+
\mathrm{Tr}\!\left(C_g + C_r - 2\left(C_g C_r\right)^{1/2}\right),
\end{equation}
where lower values indicate closer alignment between the generated and ground-truth image distributions.

\paragraph{Point-FID (3D set-level semantic similarity).}
To complement image-based fidelity with a 3D semantic metric, we report Point-FID computed from point-cloud features extracted by Point-E~\cite{alex2022pointe}. For each object, we obtain a point cloud from the predicted 3D output and a point cloud from the ground-truth 3D shape, then compute Point-E feature embeddings $\psi(\cdot)$. As with FID, we aggregate embeddings across the dataset into generated and reference distributions, estimate Gaussian statistics $(\tilde m_g,\tilde C_g)$ and $(\tilde m_r,\tilde C_r)$, and compute
\begin{equation}
\mathrm{Point\text{-}FID}
=
\|\tilde m_g-\tilde m_r\|_2^2
+
\mathrm{Tr}\!\left(\tilde C_g + \tilde C_r - 2\left(\tilde C_g \tilde C_r\right)^{1/2}\right).
\end{equation}
Lower Point-FID indicates that the predicted 3D shapes are closer to ground truth in the learned 3D semantic feature space, providing a set-level measure of 3D semantic similarity beyond per-view image agreement.

\section{More Baseline Comparisons and Efficiency Details}
\label{app:baseline-details}

\paragraph{Other baseline families.}
We evaluate several alternative routes for injecting semantic intent into image-to-3D generation: video generation followed by 3D generation, multi-view generation followed by 3D generation, commercial 2D editing followed by 3D generation, and direct text+image conditioning. The video-generation baseline uses Wan2.2-TI2V~\cite{wan2025} to generate orbit-style frames from the observation and intent prompt, followed by TRELLIS multi-image reconstruction~\cite{xiang2024trellis}. The multi-view-generation baseline uses MV-Adapter~\cite{Huang_2025_ICCV} with the same reconstruction protocol. The commercial 2D-editing baseline replaces SDXL with Nano Banana Pro~\cite{google2025nanobananapro} before 3D generation. The direct text+image baseline uses TRELLIS with an image-conditioned observation branch and a text-conditioned semantic branch combined by CFG-style interpolation.

Tables~\ref{tab:appen-baseline-occ} and~\ref{tab:appen-baseline-unseen} show that these alternatives do not jointly preserve the observation and maintain 3D consistency. In particular, \methodabbr achieves the best CLIP$_\text{img}$ and Point-FID on \extremeocc and the best CLIP$_\text{img}$ on \unseen, while maintaining competitive CLIP$_\text{txt}$. Figures~\ref{fig:comp-trellis}, \ref{fig:comp-wan}, and~\ref{fig:comp-nbp} further illustrate the failure modes: direct text+image fusion can misalign branches, video/multi-view generation can produce inconsistent reconstruction evidence, and non-3D-aware commercial editing can introduce geometry artifacts or attribute leakage.

\begin{table}[h]
\centering
\caption{\textbf{Comparison with other baseline families on \extremeocc.} These baselines cover video generation, multi-view generation, direct text+image conditioning, and commercial 2D editing pipelines.}
\label{tab:appen-baseline-occ}
\small
\setlength{\tabcolsep}{3pt}
\renewcommand{\arraystretch}{1.12}
\begin{tabular}{l|ccccc}
\toprule
Method & CLIP$_\text{img}$$\uparrow$ & CLIP$_\text{txt}$$\uparrow$ & FID$\downarrow$ & LPIPS$\downarrow$ & Point-FID$\downarrow$ \\
\midrule
Wan2.2-TI2V + TRELLIS & 0.78 & 22.98 & 107.49 & 0.85 & 110.3 \\
MV-Adapter + TRELLIS & 0.80 & 22.70 & 148.66 & 0.81 & 238.3 \\
TRELLIS (text+image) & 0.82 & 26.36 & 116.51 & 0.79 & 137.7 \\
Nano Banana Pro + 3D & 0.83 & \textbf{27.49} & 77.68 & 0.77 & 88.6 \\
\textbf{Ours} & \textbf{0.87} & 27.26 & \textbf{39.44} & \textbf{0.51} & \textbf{81.1} \\
\bottomrule
\end{tabular}
\end{table}

\begin{table}[h]
\centering
\caption{\textbf{Comparison with other baseline families on \unseen.} Because \unseen has multiple valid completions and no unique 3D ground truth, we report CLIP-based observation and text alignment.}
\label{tab:appen-baseline-unseen}
\small
\setlength{\tabcolsep}{8pt}
\renewcommand{\arraystretch}{1.12}
\begin{tabular}{l|cc}
\toprule
Method & CLIP$_\text{img}$$\uparrow$ & CLIP$_\text{txt}$$\uparrow$ \\
\midrule
Wan2.2-TI2V + TRELLIS & 0.80 & 26.71 \\
MV-Adapter + TRELLIS & 0.81 & 25.54 \\
TRELLIS (text+image) & 0.81 & 27.24 \\
Nano Banana Pro + 3D & 0.81 & \textbf{27.41} \\
\textbf{Ours} & \textbf{0.87} & 27.23 \\
\bottomrule
\end{tabular}
\end{table}

\begin{figure}
    \centering
    \includegraphics[width=1\linewidth]{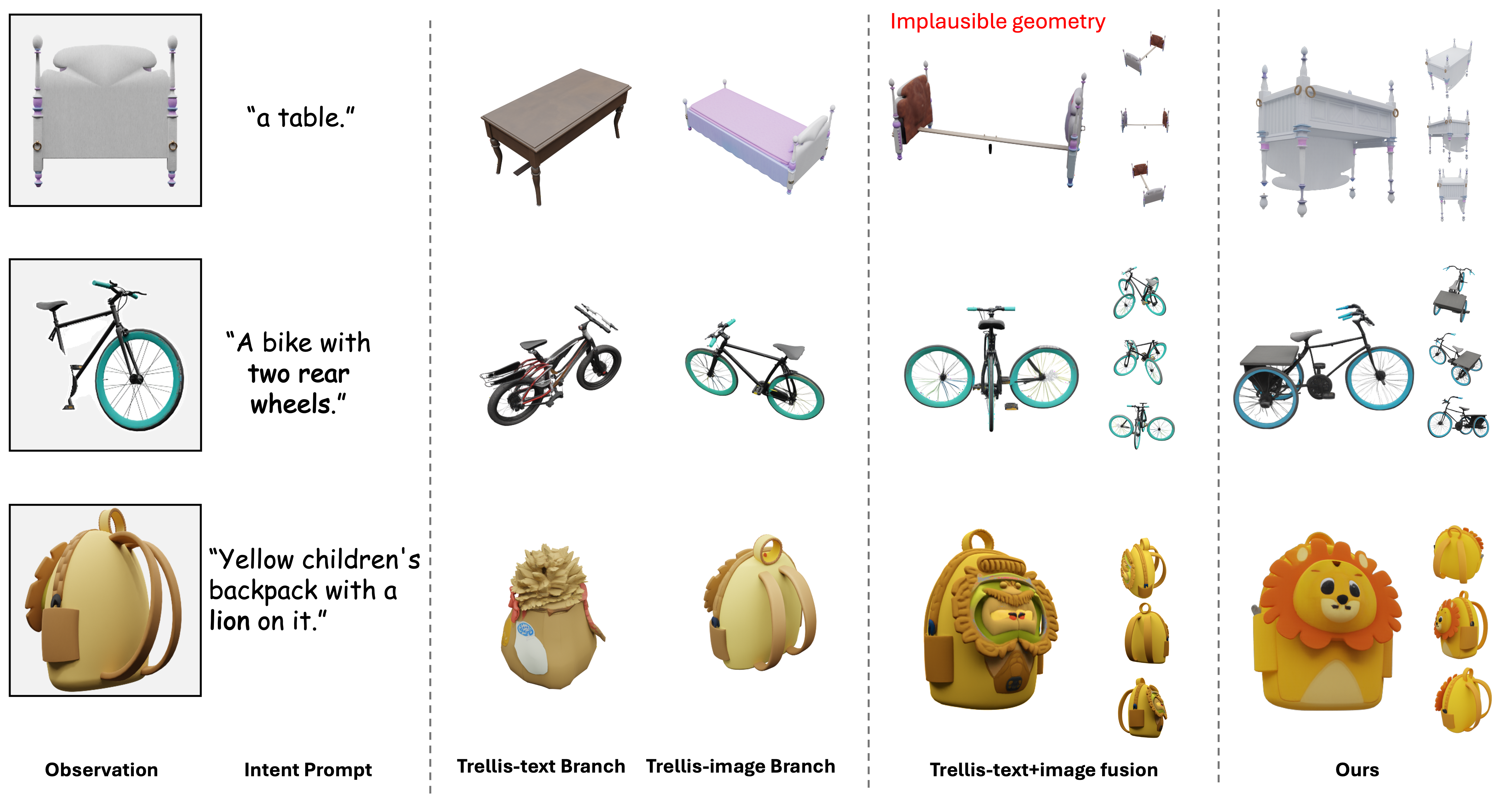}
    \caption{Qualitative comparison with the TRELLIS (text+image) fusion baseline. This baseline often produces implausible geometry because the objects generated by its text- and image-conditioned branches are not visually consistent, making spatial alignment difficult. In contrast, our method injects text conditioning as low-pass semantic guidance in the visual latent space, yielding more coherent and geometrically consistent results.
}
    \label{fig:comp-trellis}
\end{figure}

\begin{figure}
    \centering
    \includegraphics[width=1\linewidth]{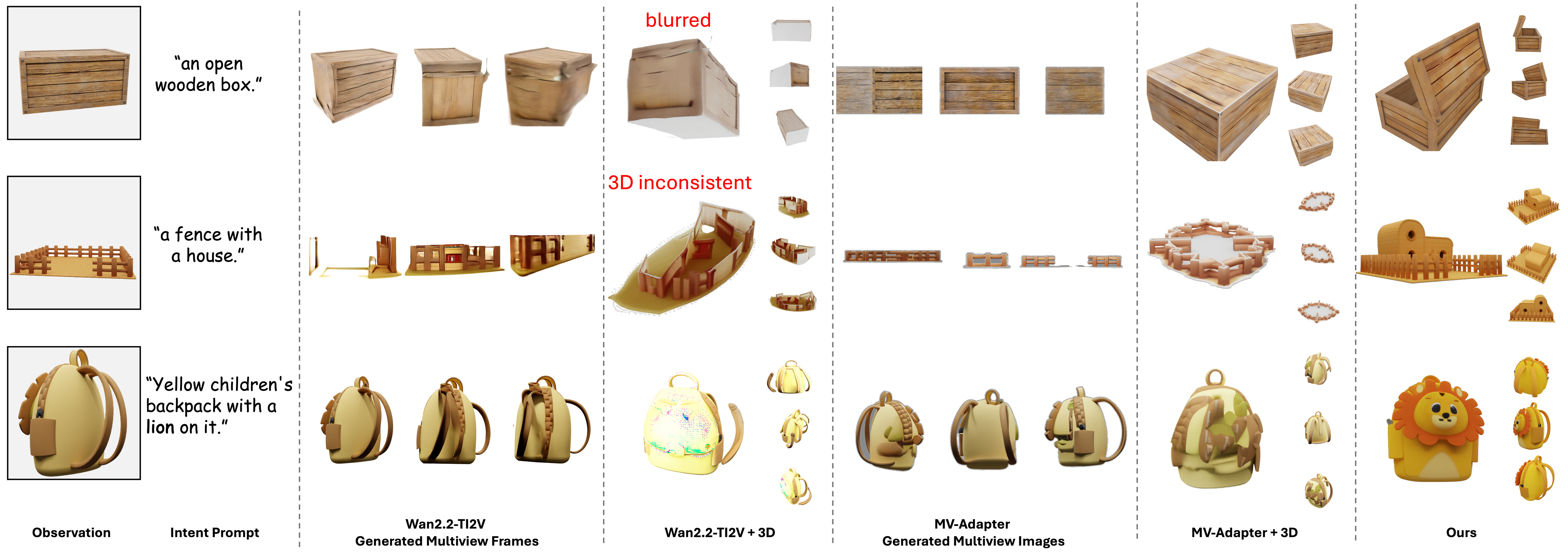}
    \caption{Qualitative comparison with Wan2.2-TI2V + TRELLIS and MV-Adapter + TRELLIS. For Wan2.2-TI2V, we use the enhanced prompt: ``A clean studio turnaround video of {intent prompt}, centered in frame, isolated from the background, with the camera orbiting smoothly around the object to reveal multiple viewpoints. The object stays fixed while only the camera moves. Consistent geometry, consistent appearance, product-shot lighting, plain neutral background, the full object remains visible in every frame, no extra objects, no scene changes, no cuts, no zoom.'' Despite this careful prompting, the video-generation pipeline often fails to produce 3D-consistent frames, leading to distorted or blurred 3D reconstructions.
The multi-view generation pipeline also struggles in this setting, often failing to preserve the observation evidence while following the text prompt, which results in distorted outputs. In contrast, our method operates during 3D generation and produces results that are both geometrically consistent and faithful to the input observation and text intent.
}
    \label{fig:comp-wan}
\end{figure}

\begin{figure}
    \centering
    \includegraphics[width=1\linewidth]{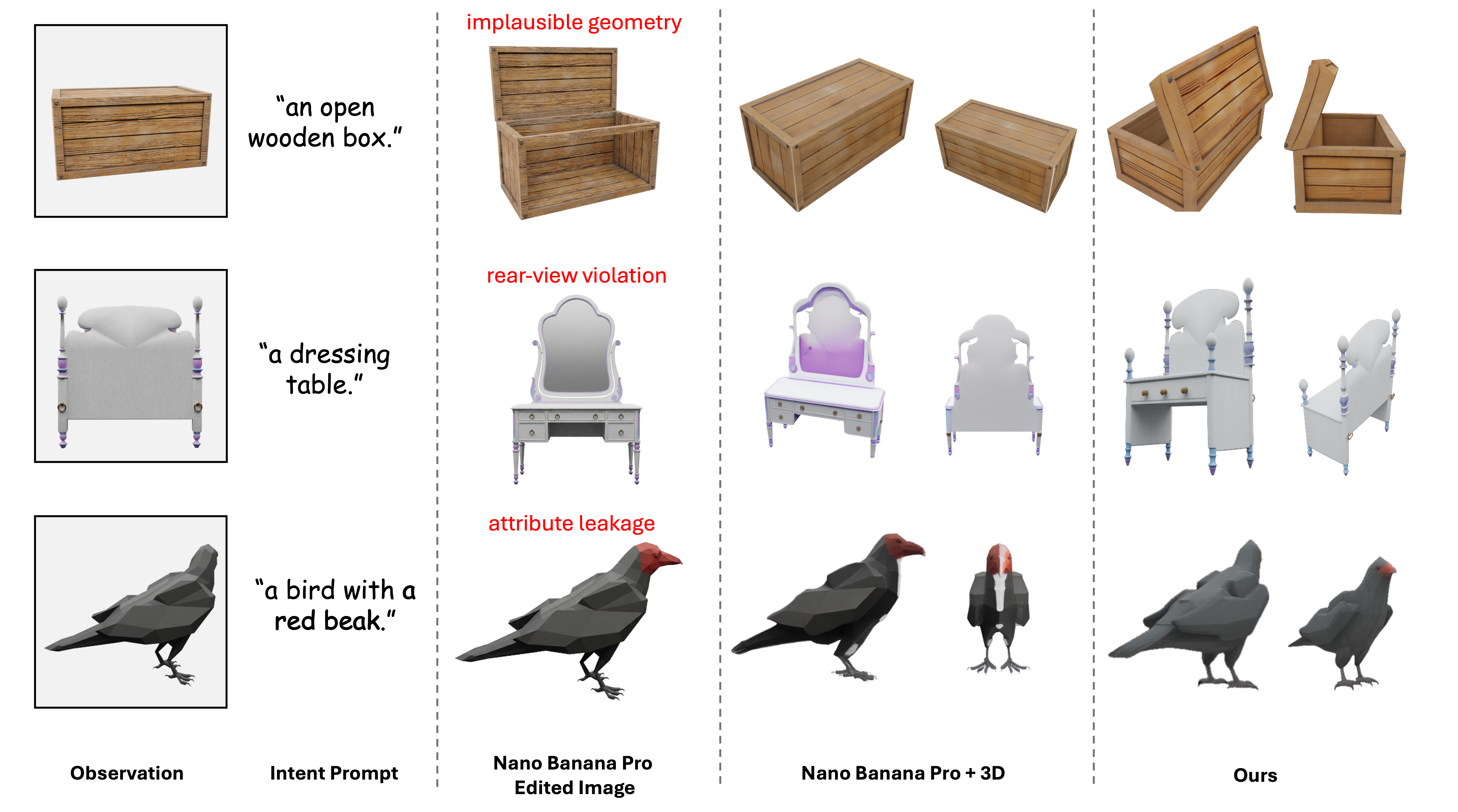}
    \caption{Qualitative comparison with Nano Banana Pro + 3D. As a non-3D-aware commercial 2D editing model, Nano Banana Pro can produce edits with implausible geometry, back-view inconsistency, and attribute leakage. In contrast, our method achieves text-guided amodal 3D generation that remains consistent with both the text intent and the observation evidence.}
    \label{fig:comp-nbp}
\end{figure}

\paragraph{Extended ambiguity evaluation.}
We evaluate a 100-case extension of \unseen (details are in Appendix~\ref{app:dataset-curation}). The original 21-case set focuses on inherent ambiguity, while the extended set also includes view-induced ambiguity. As shown in Table~\ref{tab:appen-extended-unseen}, the trend remains consistent: \methodabbr improves both observation and prompt alignment over SAM3D.

\begin{table}[h]
\centering
\caption{\textbf{Extended 100-case \unseen evaluation.}}
\label{tab:appen-extended-unseen}
\small
\setlength{\tabcolsep}{12pt}
\renewcommand{\arraystretch}{1.12}
\begin{tabular}{l|cc}
\toprule
Method & CLIP$_\text{img}$$\uparrow$ & CLIP$_\text{txt}$$\uparrow$ \\
\midrule
SAM3D & 0.839 & 25.6 \\
\textbf{Ours} & \textbf{0.845} & \textbf{26.9} \\
\bottomrule
\end{tabular}
\end{table}

\paragraph{Runtime and memory.}
We measure inference cost on a 10-sample random subset using a single NVIDIA RTX A5000 GPU. The results are shown in Table~\ref{tab:appen-efficiency}. \methodabbr adds extra branch evaluations, increasing runtime, while peak memory grows modestly. The current implementation prioritizes memory over runtime; batched branch inference and feature caching may further reduce wall-clock overhead.

\begin{table}[h!]
\centering
\caption{\textbf{Runtime and memory overhead.} Memory is peak GPU memory in GB. \methodabbr adds extra branch evaluations, increasing runtime, while peak memory grows modestly. The current implementation prioritizes memory over runtime; batched branch inference and feature caching may further reduce wall-clock overhead.}
\label{tab:appen-efficiency}
\small
\setlength{\tabcolsep}{5pt}
\renewcommand{\arraystretch}{1.12}
\begin{tabular}{ll|ccc}
\toprule
Backbone & Method & Time (s/sample) & Allocated (GB) & Reserved (GB) \\
\midrule
\multirow{3}{*}{SAM3D}
& Base & 15.59 & 18.07 & 21.85 \\
& Ours & 38.40 & 18.60 & 22.60 \\
& Overhead & 2.46$\times$ & +2.9\% & +3.4\% \\
\midrule
\multirow{3}{*}{TRELLIS}
& Base & 12.81 & 5.57 & 5.60 \\
& Ours & 23.31 & 6.02 & 6.07 \\
& Overhead & 1.82$\times$ & +8.0\% & +8.4\% \\
\bottomrule
\end{tabular}
\end{table}

\section{Visibility Mask Details and Hyperparameters}\label{app:vis-mask}

This appendix provides the concrete parameterization used to compute the visibility-aware blending mask in Sec.~\ref{sec:dualbranch}.

\paragraph{Depth map construction.}
Given the pose predicted by the backbone under $c_{\rm obs}$, we transform voxel indices $c_i\in\{0,\dots,63\}^3$ to camera space
\begin{equation}
x_i = s \odot (R c_i) + t,\qquad x_i=(x_i,y_i,z_i),
\end{equation}
and project to pixels $(u_i,v_i)$ with standard pinhole intrinsics $K=(f_x,f_y,c_x,c_y)$:
\begin{equation}
u_i = c_x - f_x \frac{x_i}{z_i}, \quad v_i = c_y - f_y \frac{y_i}{z_i}.
\end{equation}
We build a z-buffer depth map
\begin{equation}
D(u,v)=\min \{\, z_i \mid (u_i,v_i)=(u,v)\,\},
\end{equation}
and densify it by min-pooling dilation with kernel size $k$ (odd integer), producing $D'$.

\paragraph{Voxel-aware scaling.}
Let $s_{\max}$ denote the maximum component of the estimated scale $s$, and $s_{\rm res}$ the voxel resolution, we can define the object-space voxel size
$s_{\text{vox}} = {s_{\max}}/{s_{\rm res}}$.

We further denote the object depth by $z_{\text{obj}} = t_z$ and the average focal length by $f_{\text{avg}}=\frac{f_x+f_y}{2}$.
We set the depth tolerance to scale with voxel size:
\begin{equation}
\sigma_d = \beta\, s_{\text{vox}},
\end{equation}
and choose the dilation kernel based on the approximate projected voxel footprint:
\begin{equation}
k = \mathrm{odd}\!\left(\mathrm{round}\!\left(\gamma\, s_{\text{vox}}\, \frac{f_{\text{avg}}}{z_{\text{obj}}}\right)\right),
\end{equation}
where $\mathrm{odd}(\cdot)$ rounds to the nearest odd integer (and we clamp to $k\ge 1$ in implementation). Intuitively, $s_{\text{vox}}\frac{f_{\text{avg}}}{z_{\text{obj}}}$ approximates how many pixels a voxel spans under perspective projection, so $k$ adapts to both object scale and depth.

\paragraph{Soft visibility weight.}
For each voxel $i$, with projected pixel $(u_i,v_i)$ and depth $z_i$, we compute the occlusion margin
\begin{equation}
\Delta_i = z_i - D'(u_i,v_i),
\end{equation}
and convert it into a soft visibility weight
\begin{equation}
m_i = \exp\!\left(-\lambda\left(\frac{\max(\Delta_i,0)}{\sigma_d}\right)^2\right).
\end{equation}

\paragraph{Default values.}
Unless otherwise specified, we use fixed defaults across all experiments:
\begin{equation}
\beta=1.5,\qquad \gamma=1.5,\qquad \lambda=3.
\end{equation}
These settings make the visibility boundary stable under discretization and modest pose noise while remaining sufficiently selective to prevent prior overpainting on visible surfaces.

\section{More Qualitative Results}
See Fig.~\ref{fig:appen-more-qual-shield}, Fig.~\ref{fig:appen-more-qual-3}, Fig.~\ref{fig:appen-more-qual-backboard}, and Fig.~\ref{fig:appen-more-qual-4}.

\begin{figure}
    \centering
    \includegraphics[width=1\linewidth]{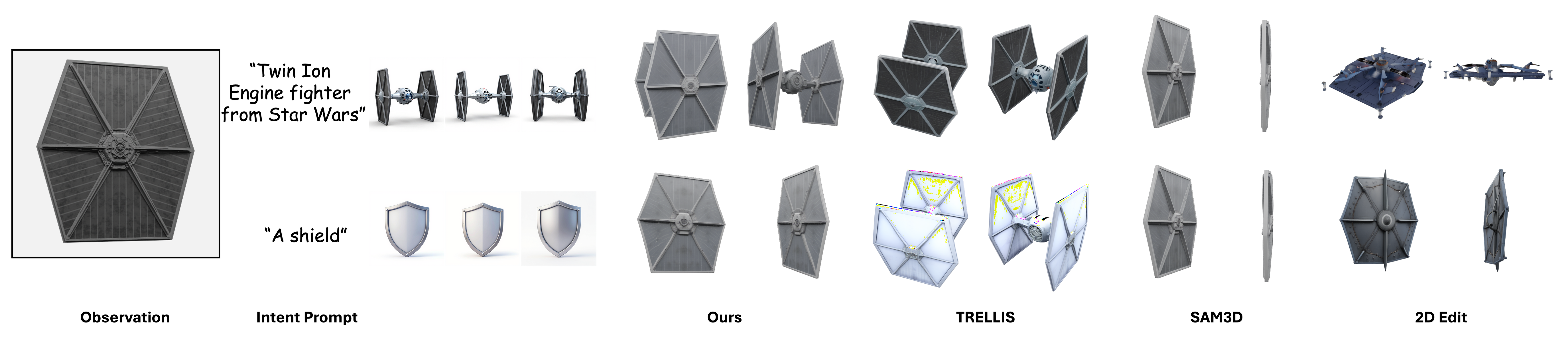}
    \caption{More qualitative results. TRELLIS and SAM3D remain overfitted to the observed geometry, producing either a shield or a TIE fighter regardless of intent, whereas our method generates either structure as specified by the intent prompt.}
    \label{fig:appen-more-qual-shield}
\end{figure}

\begin{figure}
    \centering
    \includegraphics[width=1\linewidth]{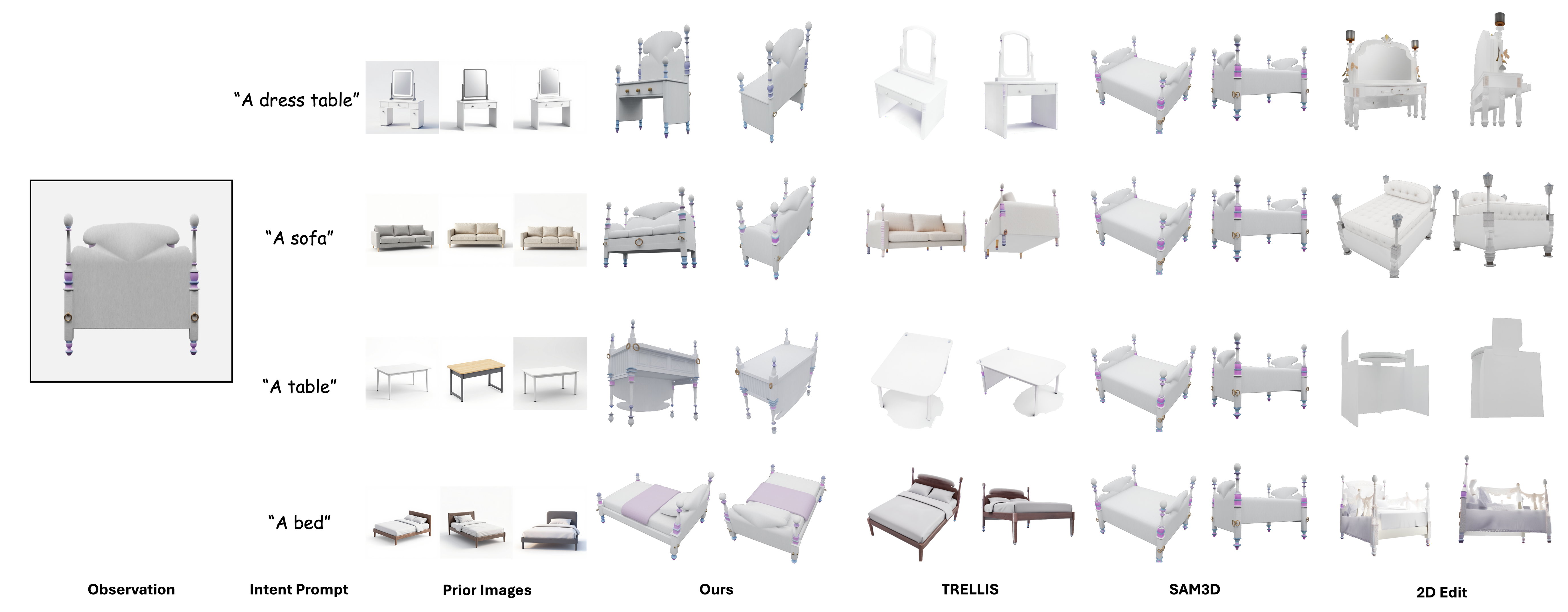}
    \caption{More qualitative results.}
    \label{fig:appen-more-qual-backboard}
\end{figure}

\begin{figure}
    \centering
    \includegraphics[width=1\linewidth]{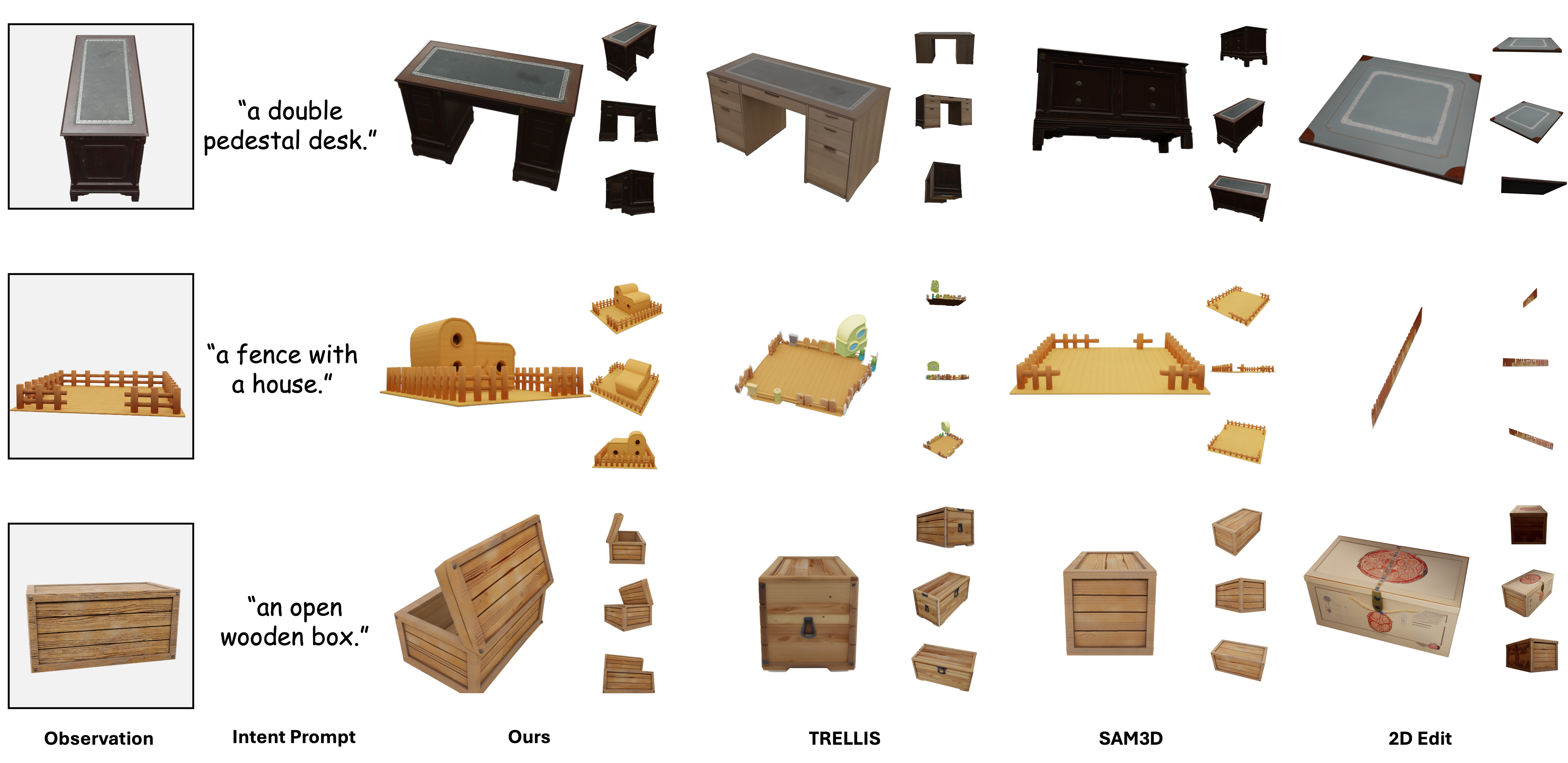}
    \caption{More qualitative results.}
    \label{fig:appen-more-qual-3}
\end{figure}

\begin{figure}
    \centering
    \includegraphics[width=1\linewidth]{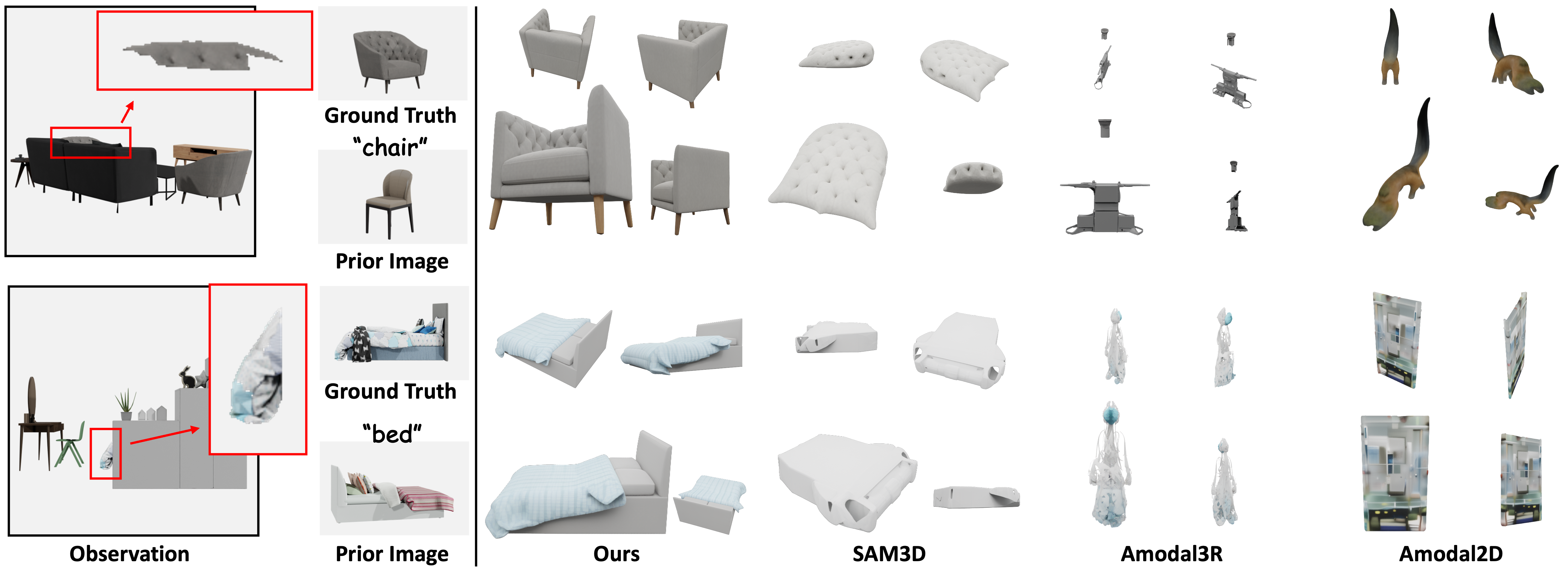}
    \caption{More qualitative results on \extremeocc.}
    \label{fig:appen-more-qual-4}
\end{figure}

\section{Failure Cases}

Our method assumes that the observation and prior are broadly compatible: the observation is ambiguous in some aspects, and the prior supplies complementary cues to resolve that ambiguity. When the two inputs contain mutually conflicting evidence (\eg, two different frontal faces with incompatible identities), the model has no consistent solution and may fail to preserve either condition faithfully. In these cases, our method tends to act as a soft ``concept fusion'' mechanism, interpolating between the two conditions due to their asymmetric guidance. We illustrate this behavior on cases from the morphing benchmark from Interp3D~\cite{xiaolu2026interp3d} (see Fig.~\ref{fig:failure}).

\begin{figure}[h]
    \centering
    \includegraphics[width=0.8\linewidth]{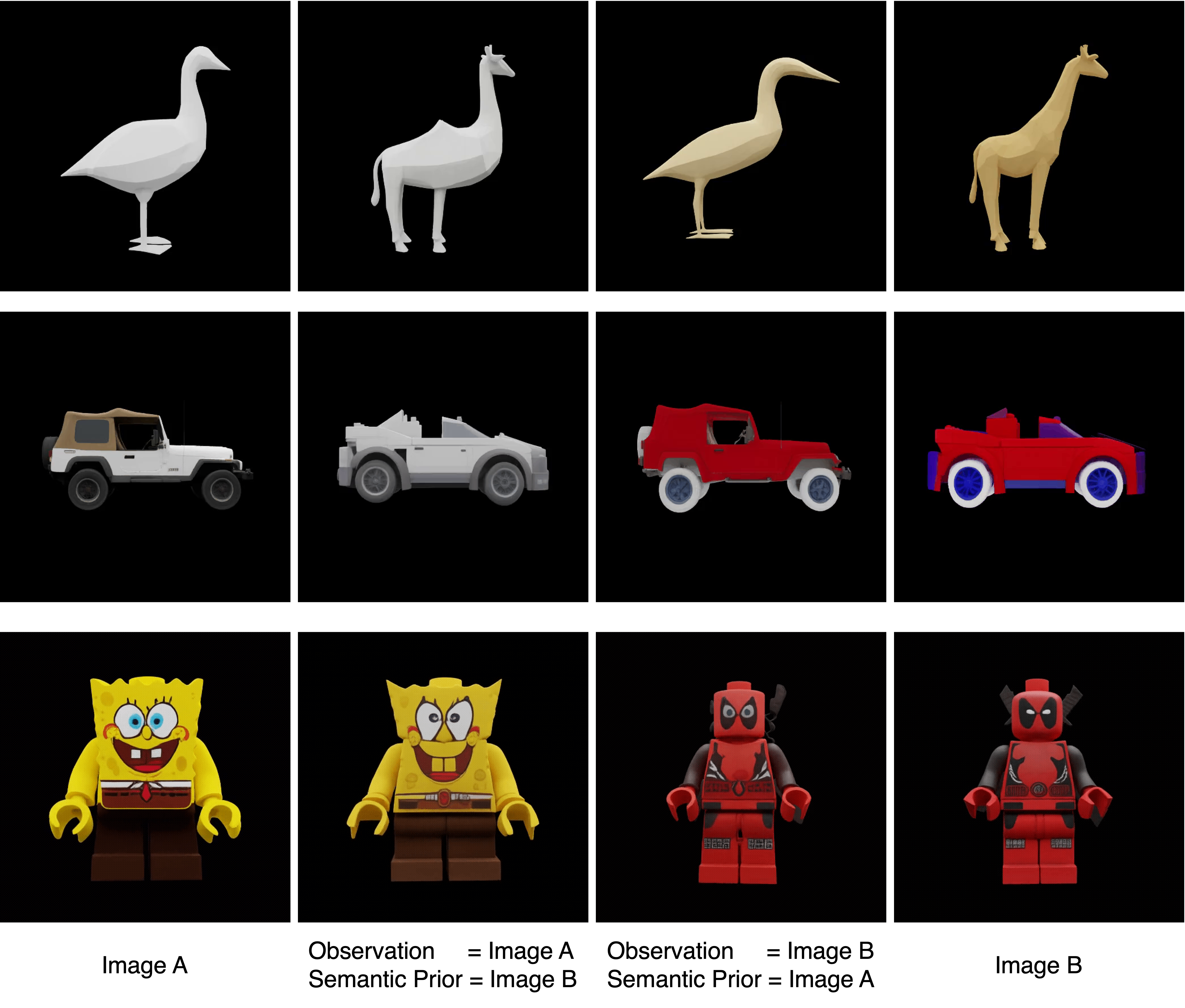}
    \caption{Failure cases under conflicting conditions. When the observation and prior disagree, no evidence-consistent solution exists, and our method may interpolate between them as a soft concept fusion due to asymmetric guidance.}
    \label{fig:failure}
\end{figure}

\end{document}